\documentclass{article}

\usepackage{iclr2026_conference, times}

\usepackage{arxiv}
\usepackage{math_headers}

\title{Stagewise Reinforcement Learning and the Geometry of the Regret Landscape}

\newcommand\equalContribution{=}

\author{%
    Chris Elliott\textsuperscript{\equalContribution}\\
    Timaeus
\And
    Einar Urdshals\textsuperscript{\equalContribution}\\
    Timaeus
\And
    David Quarel\textsuperscript{\equalContribution}\\
    Timaeus
\AND
    Matthew Farrugia-Roberts\\
    University of Oxford
\And
    Daniel Murfet\\
    Timaeus
}

\begin{document}

\maketitle
\makeatletter
\def\@makefnmark{\hbox{\@textsuperscript{\normalfont\@thefnmark}}}
\makeatother
\renewcommand{\thefootnote}{=}
\setcounter{footnote}{0}
\footnotetext{These authors contributed equally to this work.}
\setcounter{footnote}{0}
\renewcommand{\thefootnote}{\arabic{footnote}}

\begin{abstract}
Singular learning theory characterizes Bayesian learning as an evolving tradeoff between accuracy and complexity, with transitions between qualitatively different solutions as sample size increases. We extend this theory to reinforcement learning, proving that the concentration of a generalized posterior over policies is governed by the local learning coefficient (LLC), an invariant of the geometry of the regret function. This theory predicts that deep reinforcement learning with SGD should proceed from simple policies with high regret to complex policies with low regret. We verify this prediction empirically in a gridworld environment exhibiting stagewise policy development: phase transitions over training manifest as ``opposing staircases'' where regret decreases sharply while the LLC increases.
\end{abstract}

\section{Introduction}

What principles determine the kinds of agents that deep reinforcement learning produces? This question grows increasingly salient as these principles shape the agents that play ever more integral roles in our society with each passing year. This question also presents a fundamental challenge for deep learning science, since its answering requires understanding the inductive biases of the kinds of extremely complex high-dimensional stochastic optimization processes that comprise neural network training.

In light of the formidable nature of this challenge, we propose the following approach to an \emph{empirical} understanding of the inductive biases of deep RL algorithms: First, we should theoretically understand the inductive biases of an idealized learning process chosen to capture core elements of deep RL, while still being analytically tractable. Second, we should derive predictions about the kinds of phenomena we would see in deep RL if it were driven by similar principles. Finally, we should experimentally test these predictions in order to validate the idealized learning process as a partial model of the real deep RL process.

A similar approach has recently proven fruitful in the science of deep \emph{supervised} learning \citep[e.g.,][]{Chen+2023, Hoogland+2025, TransientRidge}. These works use singular Bayesian inference as an idealized learning process, applying Watanabe's free energy formula for Bayesian inference in singular statistical models \citep{WatanabeGrey, WatanabeGreen} to frame learning as an evolving tradeoff between accuracy (as measured by loss) and complexity (as measured by the \emph{local learning coefficient,} LLC, a geometric invariant of the loss landscape; \citealp{LLC}). With this understanding, they predict and experimentally verify the existence of various stagewise learning phenomena in deep learning and their connection to loss landscape geometry.

In this paper, we extend this approach to deep RL, offering the following contributions:
\begin{enumerate}
    \item
        In \cref{section:setting}, we develop our idealization of the deep RL process, namely \textbf{a generalization of singular Bayesian inference} that accounts for the role played by rewards in RL using Kalman--Todorov duality \citep{kalman1960new, Todorov, Levine}, as well as accounting for potential non-stationarity in experience data using an importance sampling objective.
    
    \item
        In \cref{section:theory}, we outline \textbf{a generalization of Watanabe's free energy formula and Bayesian information criterion} for singular statistical models, extending the main results of singular learning theory to our generalized inference setting. The proof follows \citet{WatanabeGrey, WatanabeWBIC, WatanabeGreen}, with the main extensions being the application of Lyapounov's central limit theorem and the identification of a new fundamental condition.
        \Cref{theory_appendix} gives proofs of these results in a unified setting.
    
    \item
        In \cref{section:predictions}, we leverage these new results to derive \textbf{qualitative predictions about deep RL.} In particular, we predict that deep RL should proceed stagewise from simple, high-regret policies to complex, low-regret policies, with policy complexity measured by the LLC.

    \item
        In \cref{section:empirical_results}, we \textbf{validate these predictions in deep reinforcement learning experiments} in a simple gridworld environment. \Cref{fig:LLC_example} previews the ``opposing staircases'' phenomenon, where we see the policy undergoes stagewise learning with decreasing regret and increasing complexity as measured by LLC.
\end{enumerate}

Our results demonstrate that deep reinforcement learning shares a fundamental property of Bayesian inference: the agents produced by learning are not necessarily optimal according to the return/regret. Rather, deep reinforcement learning agents exhibit \emph{stagewise learning governed by an evolving tradeoff between performance and simplicity.}
This learning dynamic has important consequences for practical applications of deep reinforcement learning in current and future agentic AI systems. We discuss these implications in \cref{section:relevance_alignment}.
Our results also underscore the LLC as a principled model complexity measure for RL, we survey related complexity measures in \cref{section:complexity_in_RL}.
Finally, our results provide a stepping stone for further theoretical analysis of the dynamics of deep RL, and invite the development of SLT-based interpretability and training tools, which we discuss in \cref{section:future_work}.

\begin{figure}[t]
\centering
\includesvg[width=0.9\linewidth]{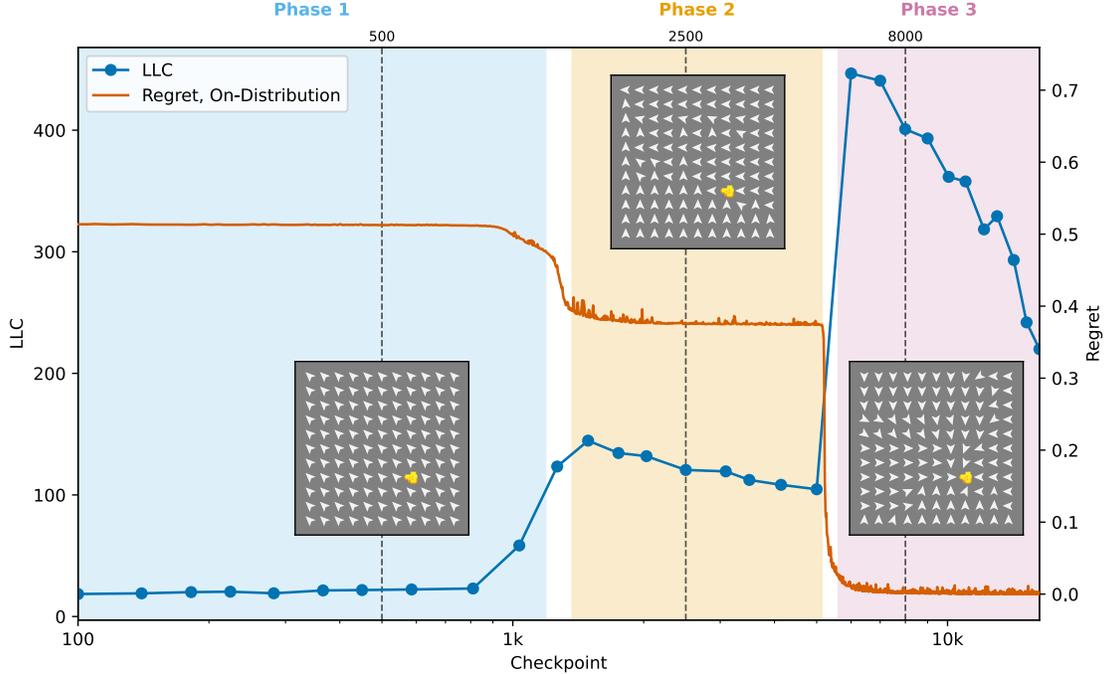}
\caption{\textbf{Opposing staircases of regret and complexity}: Comparison of the regret and complexity (as estimated by the local learning coefficient) across training for an agent optimized to follow shortest paths towards the goal location (cheese), which was located in the top-left corner in $\sim\frac 13$ of sampled episodes and uniformly across all locations in the remaining episodes. Phases are indicated along with visualizations of policies at the indicated checkpoints (i.e., policy gradient steps). In the first phase the agent moves up and left with equal probability; in the second phase it moves deterministically towards the top-left region, passing through the goal if possible, and finally in the third phase it moves directly toward the goal. 
For each location $s$, the arrow drawn points in the direction of the vector 
$v_s = [\pi(\rightarrow | s, \varnothing) - \pi(\leftarrow | s, \varnothing), 
\pi(\uparrow | s, \varnothing) - \pi(\downarrow | s, \varnothing)]$ and has size proportional to $||v_s||_1$ representing the expected direction of movement if the agent
spawns in cell $s$. Policy representation taken from \citet{mini2023understanding}. Phase transitions are associated with rapid decreases in regret and rapid increases in the LLC estimate.
}
\label{fig:LLC_example}
\end{figure}

\section{Background}

In this section, we provide a brief primer on singular learning theory (SLT) for supervised learning, and introduce preliminary definitions for discussing reinforcement learning (RL).

\subsection{Bayesian Deep Supervised Learning and Singular Learning Theory}

When we refer to Bayesian inference as a learning process, we mean how the posterior distribution $p(w|D_n)$ over parameters $w \in W$ given data $D_n$ evolves as $n$ increases.

In supervised learning for \emph{regular} statistical models, where the data consists of $n$ random variables sampled independently and identically and the Fisher information matrix is invertible, the story is well-known and simple: as $n \rightarrow \infty$, the posterior converges to a normal distribution centered on the maximum likelihood estimator with variance controlled by the (inverse of the) Fisher information matrix \citetext{by the Bernstein-von Mises theorem, see \citealp{hartigan1983asymptotic}}.

However, when we parametrize our statistical model by a neural network the Fisher information matrix is not invertible and the model is \emph{not} regular. This is the typical situation in deep learning \citep{Hagiwara+1993, watanabe2007almost, wei2022deep}. Statistical models which are not regular are called \emph{singular} and the appropriate mathematical framework for studying the Bayesian learning process for supervised learning is singular learning theory \citep[SLT;][]{WatanabeGrey}, which leverages algebraic geometry to describe Bayesian learning dynamics in terms of the nature of singularities in the negative log likelihood.

In particular, SLT gives an asymptotic expansion of the local free energy, known as \emph{Watanabe's free energy formula,} which we informally describe now.
Let $u \in W$ be a local minimum of the expected negative log likelihood, and let $U$ be a ball around $u$ in which $u$ is a global minimum. Then, under certain technical conditions on the statistical model, we have an asymptotic expansion in $n$ \citetext{\citealp[Thm.\ 11]{WatanabeGreen}; \citealp{LLC}} of the free energy associated to this neighborhood:
\begin{equation*}
    -\log \int_U \exp(-nL_n(w)) \phi(w) \d w = nL_n(u) n + \lambda(U) \log n - (m(U)-1) \log \log n + O_{\mr P}(1)
\end{equation*}
where $L_n(u)$ is the negative log likelihood of $D_n$ at $u$ and $\lambda(U)$ and $m(U)$ are coefficients derived from the geometry of the log likelihood around $u$. The coefficient $\lambda(U)$ is also known as the \emph{local learning coefficient} \citep[LLC;][]{LLC}. The lower-order terms include, for example, a contribution from the prior $\phi$.

As we will discuss in later sections, this formula implies that the learning dynamics of Bayesian inference with singular models (including Bayesian deep supervised learning) are \emph{much richer} than in the regular case.
For example, rather than converging to a normal distribution as in the regular case, in singular models the posterior can ``jump'' at certain critical sample sizes $n^*$ from concentrating in one region of parameter space for $n < n^*$ to concentrating in another region for $n > n^*$. This behavior resembles phase transitions in statistical physics \citep[\S 7.6]{WatanabeGrey}, and these phase transitions are characterized by an evolving tradeoff between how well a model fits the data as measured by the log likelihood, and its complexity as measured by the LLC \citetext{\citealp[\S7.6]{WatanabeGrey}; \citealp{WatanabeWBIC,Chen+2023}}. This theory has been verified for Bayesian learning in small models \citep{Chen+2023} and is conjecturally related to stagewise learning in SGD training in a range of systems including transformers \citep{Hoogland+2025, TransientRidge}.

So much for supervised learning. The main aim of this paper is to extend the above story to a more general setting that captures the core features of deep RL, including (1)~the essential role of singularities, (2)~the use of rewards rather than labels, and (3)~the non-stationarity of experience data.

\subsection{Reinforcement Learning Preliminaries}\label{section:rl_preliminaries}

We now introduce the preliminary definitions we require for describing our results in the main text. Note that in \cref{theory_appendix} we treat a more general setting, relaxing some of the restrictions introduced here for brevity (such as finiteness of the Markov decision process).

We will study a partially observable finite Markov decision process with sets $\mc S, \mc A$ and $\mc O$ of states, actions and observations respectively.  Denote the transition function of the Markov problem by\footnote{One usually imposes restrictions on the transition function $p$, in particular the condition that the observation depends only on the output state.}
\[p \colon \mc S \times \mc A \to \Delta(\mc S \times \mc O \times \RR)\]
 and suppose that the expected value of the return associated to each state-action pair is finite.  Fix a maximum episode length $T_{\mr{max}}$.  A \emph{trajectory} is a finite sequence
 \[\tau = (s_0, o_0, a_1, s_1, o_1, \ldots, a_{T_{\mr{max}}}, s_{T_{\mr{max}}}, o_{T_{\mr{max}}})\]
where $s_i \in \mc S, o_i \in \mc O$ and $a_i \in \mc A$.  We write $\mc T$ for the finite set consisting of all trajectories.

Let $\gamma \in [0,1]$ denote the discount factor.  Let $r(s,a)$ denote the expected reward in state $s \in \mc S$ for action $a \in \mc A$.  The \emph{return of a trajectory} $\tau \in \mc T$ is given by
\[r(\tau) = \sum_{i=1}^{T_{\mr{max}}} \gamma^i r(s_{i-1},a_i),\]
where we note the reuse of the symbol $r$, this will not lead to ambiguity since we will not need to refer to the expected return of a state-action pair again.

Consider a real analytic family $\{\pi_w\}$ of policies for the Markov problem parameterized by $w \in W$. This family, together with the transition function $p$ and a distribution $\Lambda$ over initial states, determines a family $\{q_w\}$ of probability distributions on the set $\mc T$ of trajectories defined by
 \[q_{w}(\tau) = \Lambda(s_0)p(o_0|s_0)\prod_{i=1}^{T_{\mr{max}}} \pi_w(a_i|o_{i-1})p(s_i|s_{i-1},a_i)p(o_i|s_i).\]
We can now define our optimization objective: the \emph{expected return} at parameter $w \in W$ is $R(w) = \bb E_{\tau \sim q_w}(r(\tau))$. Let $R_{\mr{max}} = \sup_{w \in W} R(w)$.  The \emph{regret} at parameter $w \in W$ is
\[
 G(w) = R_{\mr{max}} - R(w)\,.
\]
Minimizing regret is equivalent to maximizing return. For simplicity, we frame our theory around minimizing regret. It will be convenient to write the regret in the alternative form $\bb E_{\tau \sim q_w}(g(\tau))$ where $g(\tau) = V^*(s_0) - r(\tau)$ and $V^*(s_0)$ denotes the optimal expected return from the initial state $s_0$ of $\tau$.  Since $s_0 \sim \Lambda$ under every $q_w$, the expected value $\bb E_{\tau \sim q_w}(V^*(s_0)) = \bb E_{s_0 \sim \Lambda}(V^*(s_0))$ does not depend on $w$, so $\bb E_{\tau \sim q_w}(g(\tau)) = G(w) + c$ where $c$ is a non-negative constant.  We will assume that the model contains an optimal policy, so that $c = 0$ and $G(w) = \bb E_{\tau \sim q_w}(g(\tau))$.  Using this baseline $V^*(s_0)$ in place of the constant $R_{\mr{max}}$ will be necessary to allow our key assumption \ref{key_assumption_main_text} to hold even when the initial state is random.

\section{Bayesian Deep Reinforcement Learning}\label{section:setting}

In this section, we develop our idealized learning process, a generalization of Bayesian inference that accounts for the role of rewards and the non-stationarity of experience data in reinforcement learning. \Cref{section:generalized_posterior} culminates in \cref{posterior_def}, giving the form of the generalized Bayesian posterior (or ``Gibbs posterior''). \Cref{section:off_policy_learning} discusses the relationship between this idealized learning process and SGD-based deep RL.

\subsection{The Generalized Posterior}\label{section:generalized_posterior}

In reinforcement learning the goal of a learner is to acquire information about the environment in order to \emph{act} optimally, which means maximizing the reward $R(w)$. In Bayesian statistics the goal of a learner is to \emph{predict} and it is \emph{a priori} unclear how to incorporate scalar rewards into this framework. In addition to this problem we also need to specify what kind of observations about the environment we are learning from. 

In this paper we adopt the following two positions:
\begin{enumerate}
\item \textbf{Reward versus probability}: we follow the \emph{Kalman--Todorov} duality \citep{kalman1960new,Todorov}, which relates optimal control to optimal prediction under the following conceptual relation
\begin{equation}\label{eq:kalman_todorov}
\text{Probability} \propto e^{-\text{Cost}}\,.
\end{equation}
This principle has been studied more recently under the name ``control as inference'' \citep{Levine}.
\item \textbf{Nature of observations}: an agent learns from observing the consequences of its actions, and thus it is obvious that \emph{trajectories} should be the ``samples'' in a Bayesian approach to RL. However, unlike supervised learning, these observations are tied to the learner whose policy generates the actions. Moreover, it seems reasonable to acquire information about the environment from the experience of similar agents (including the agent itself, earlier in the learning process). Accordingly, we take
\[
D_n \;=\; \{(w_i,\tau_i)\}_{i=1}^n,
\qquad
\tau_i \sim q_{w_i}
\]
as our model of the data set where $w_1,\ldots,w_n$ is any sequence of points in $W$ and $\tau_1, \ldots, \tau_n$ is a sequence of random variables in $\mc T$ where $\tau_i$ is sampled from the distribution $q_{w_i}$. This formulation encompasses both off-policy data (distinct $w_i$) and on-policy data (all $w_i = w_0$ for a fixed $w_0$). Note that these random variables are independent but not identical, since we allow $w_i \neq w_j$.
\end{enumerate}
Given a prior $\phi(w)$ over parameters for policies, to complete a Bayesian description of reinforcement learning it only remains to define how this distribution is updated given observations $D_n$. That is, we need to define a \emph{generalized posterior} $p(w|D_n)$. Since $g(\tau)$ is our cost function and the $\tau_i$ are independent, \eqref{eq:kalman_todorov} gives a likelihood in the on-policy case (where $w_i = w_0$ for $1 \le i \le n$)
\[
p(D_n|w_0) = \prod_{i=1}^n p(\tau_i|w_0) = \prod_{i=1}^n e^{-g(\tau_i)} = e^{-n G_n(w_0)}
\]
where $G_n(w_0) = \frac{1}{n} \sum_{i=1}^n g(\tau_i)$. However we can't use Bayes rule to obtain a posterior $p(w|D_n)$, because we can only define a likelihood $p(D_n|w)$ this way for $w = w_0$. To proceed note that $G_n(w_0)$ is a Monte Carlo estimator for
\[
G(w_0) \;=\; \bb E_{\tau\sim q_{w_0}}\bigl[g(\tau)\bigr].
\]
The Kalman--Todorov principle \eqref{eq:kalman_todorov} suggests that we should assign likelihood
\begin{equation}\label{eq:likehood_candidate}
p(D_n| w) \propto \exp(-n\,\widehat G_n(w)),
\end{equation}
where $\widehat G_n(w)$ is an empirical estimate of the population regret
\[
G(w) = \bb E_{\tau\sim q_w}\bigl[g(\tau)\bigr]
= \sum_{\tau\in\mc T} q_w(\tau)\, g(\tau).
\]
The issue is that, when $\tau_i\sim q_{w_i}$, the naive average $\frac1n\sum_i g(\tau_i)$ estimates
$\frac1n\sum_i \bb E_{\tau\sim q_{w_i}}[g(\tau)]$, which depends on the behavior policies and is not, in general, equal to $G(w)$ for a counterfactual candidate $w$. To address this, fix $i$ and suppose $q_w$ is absolutely continuous with respect to $q_{w_i}$ (that is,
$q_{w_i}(\tau)=0$ implies $q_w(\tau)=0$ for all $\tau\in\mc T$), so that the Radon--Nikodym derivative
$\frac{\d q_w}{\d q_{w_i}}(\tau)$ is well-defined. Since $\mc T$ is finite in our setting, this derivative is simply the likelihood ratio $\frac{q_w(\tau)}{q_{w_i}(\tau)}$. Then for any function $f\colon\mc T\to\RR$,
\[
\bb E_{\tau\sim q_w}\bigl[f(\tau)\bigr]
=
\bb E_{\tau\sim q_{w_i}}
\left[
\frac{q_w(\tau)}{q_{w_i}(\tau)}\, f(\tau)
\right].
\]
Applying this with $f=g$ yields
\[
G(w)
=
\bb E_{\tau\sim q_w}\bigl[g(\tau)\bigr]
=
\bb E_{\tau\sim q_{w_i}}
\left[
\frac{q_w(\tau)}{q_{w_i}(\tau)}\, g(\tau)
\right].
\]
In other words, if $\tau_i\sim q_{w_i}$ then the random variable
\[
\frac{q_w(\tau_i)}{q_{w_i}(\tau_i)}\, g(\tau_i)
\]
is an unbiased estimator of $G(w)$. Averaging across independent (but not identically distributed) samples therefore yields the canonical importance-sampling
estimator
\begin{equation}\label{eq:IS_regret_estimator_here}
G_n(w)
\;=\;
\frac1n\sum_{i=1}^n
\frac{q_w(\tau_i)}{q_{w_i}(\tau_i)}\, g(\tau_i),
\qquad
\text{for which}\quad
\bb E\!\left[G_n(w)\mid w_1,\ldots,w_n\right] \;=\; G(w).
\end{equation}
Observe that $G_n(w)$ can be computed without knowledge of the environment's transition function: writing
\[
q_w(\tau)
=
q_{\mathrm{env}}(\tau)\,\prod_{t=1}^{T_{\max}} \pi_w(a_t | o_{t-1}),
\]
where $q_{\mathrm{env}}(\tau)$ denotes the product of all policy-independent environment terms, we obtain
\[
\frac{q_w(\tau)}{q_{w_i}(\tau)}
=
\prod_{t=1}^{T_{\max}}
\frac{\pi_w(a_t | o_{t-1})}{\pi_{w_t}(a_t | o_{t-1})}.
\]
Thus the importance weight can be computed from action probabilities under the candidate and behavior policies; all environment-dependent factors cancel. Combining \eqref{eq:likehood_candidate} with \eqref{eq:IS_regret_estimator_here} leads to the definition
\[
p(w | D_n)
= \frac{1}{Z}
\exp\bigl(-n G_n(w)\bigr)\phi(w)
\]
where the normalizing constant $Z$ is chosen so that this is a probability distribution. Note that this reduces to the on-policy calculation above when $w_i=w_0$ for all $i$. We will also need to introduce a tempered version of this posterior, and notation for the probability of an open set according to the posterior:

\begin{definition} \label{posterior_def}
 Fix a constant $\beta > 0$ and a positive analytic prior distribution $\phi$ on $W$.  The \emph{generalized tempered posterior distribution} (or Gibbs posterior) is the probability distribution on $W$ defined by
 \begin{align*}
  \mu_{n}^{\beta}(U) &= \frac {Z_{n,\beta}(U)}{Z_{n,\beta}(W)} \\
  \text{where } Z_{n,\beta}(U) &= \int_U \exp(-n\beta G_n(w)) \phi(w) \d w
 \end{align*}
 for open sets $U \sub W$.  We call $Z_{n,\beta}(U)$ the \emph{evidence} of the subset $U$ and $F_{n,\beta}(U) = -\log Z_{n,\beta}(U)$ the \emph{free energy}. Both of these quantities measure the concentration of the generalized posterior in $U$. 
\end{definition}

For more discussion of the generalized posterior see \cref{section:slt_gen_bayes}.

\subsection{Relationship to Deep Reinforcement Learning}
\label{section:off_policy_learning}

A subtlety in relating the generalized posterior to practical RL training is the appearance of the sequence of behavior parameters $w_1,\dots,w_n$ in the data set $D_n=\{(w_i,\tau_i)\}_{i=1}^n$. From a practical off-policy RL perspective, there is a standard answer to ``where do the $w_i$ come from?'': they are the (possibly stale) policy parameters that were used to \emph{collect} the experience currently being used for learning.

Concretely, modern off-policy RL algorithms decouple data collection from parameter updates by maintaining a replay buffer (or a distributed set of actors) containing trajectories or transitions generated under a sequence of behavior policies.
If at training time-step $k$ the learner has parameters $w^{(k)}$ and collects a batch of trajectories
$\tau^{(k)}_1,\dots,\tau^{(k)}_{B_k}\sim q_{w^{(k)}}$, then the accumulated data set over training naturally has the form
\[
\bigl\{(w^{(k)},\tau^{(k)}_j)\;:\; k\le K,\; 1\le j\le B_k \bigr\},
\]
i.e., it is a union of samples drawn under earlier iterates of the same SGD process.\footnote{In distributed actor--learner architectures one should interpret $w^{(k)}$ as the parameters on the actor at the moment of rollout, which are typically a lagged copy of the learner parameters.}
Thus, even though our theory allows an arbitrary sequence $(w_i)_{i=1}^n$, the ``canonical'' source of such a sequence in practice is simply \emph{the learning algorithm itself}, via checkpoints or lagged actors.

Recall that the RL objective is to maximize expected return $R(w)$, equivalently to minimize regret
$G(w)=R_{\max}-R(w)$ so $\nabla_w R(w) = -\,\nabla_w G(w)$. The score-function identity gives
\begin{equation}\label{eq:score_fn_identity}
\nabla_w R(w) = \nabla_w \bb E_{\tau\sim q_w}[r(\tau)]
=
\bb E_{\tau\sim q_w} \left[r(\tau)\,\nabla_w \log q_w(\tau)\right].
\end{equation}
Using the factorization
$q_w(\tau)=q_{\mathrm{env}}(\tau)\prod_{t=1}^{T}\pi_w(a_t | o_{t-1})$
we obtain
\begin{equation}\label{eq:logq_factorizes}
\nabla_w \log q_w(\tau)
=
\sum_{t=1}^{T}\nabla_w \log \pi_w(a_t| o_{t-1}),
\end{equation}
which is the usual REINFORCE structure. Finally, by the standard ``reward-to-go'' (causality) argument one may replace
the full return $r(\tau)$ in \eqref{eq:score_fn_identity} by $r(\tau_{\ge t})$ at time $t$, yielding an equivalent estimator
\[
\nabla_w R(w)
=
\bb E_{\tau\sim q_w}\!\left[\sum_{t=1}^{T} r(\tau_{\ge t})\,\nabla_w \log \pi_w(a_t| o_{t-1})\right],
\]
which matches the on-policy training rule used in this paper (see \eqref{eq:training_gradient}).

Now consider the off-policy setting where trajectories are generated under behavior policies $\pi_{w_i}$ and we wish to improve a \emph{target} policy $\pi_w$. Applying the same change-of-measure identity used above
\eqref{eq:score_fn_identity} becomes
\[
\nabla_w R(w)
=
\bb E_{\tau\sim q_{w_i}}\left[
\frac{q_w(\tau)}{q_{w_i}(\tau)}\, r(\tau)\,\nabla_w \log q_w(\tau)
\right],
\]
and substituting \eqref{eq:logq_factorizes} and the reward-to-go form yields the canonical off-policy REINFORCE estimator
\begin{equation}\label{eq:offpolicy_reinforce_traj_ratio}
\nabla_w R(w)
=
\bb E_{\tau\sim q_{w_i}}\left[
\frac{q_w(\tau)}{q_{w_i}(\tau)}
\sum_{t=1}^{T} r(\tau_{\ge t})\,\nabla_w \log \pi_w(a_t| o_{t-1})
\right].
\end{equation}
The connection to our empirical regret estimator is immediate. Differentiating $G_n$ gives
\begin{align*}
\nabla_w G_n(w)
&=
\frac1n\sum_{i=1}^n
\frac{q_w(\tau_i)}{q_{w_i}(\tau_i)} g(\tau_i)\,\nabla_w \log q_w(\tau_i)
\\
&=
\frac1n\sum_{i=1}^n
\frac{q_w(\tau_i)}{q_{w_i}(\tau_i)}
\sum_{t=1}^{T_i}
g(\tau_i) \nabla_w \log \pi_w(a_{i,t}| o_{i,t-1}),
\end{align*}
and replacing $g(\tau_i)$ by its reward-to-go form yields the trajectory-sampled analogue of \eqref{eq:offpolicy_reinforce_traj_ratio} (up to the overall sign corresponding to maximizing return versus minimizing regret). Here $g(\tau_i) = V^*(s_0) - r(\tau_i)$, where $s_0$ is the initial state of the trajectory $\tau_i$.  This is the per-trajectory regret that we introduced in \cref{section:rl_preliminaries} -- the baseline term given by the initial state value vanishes in expectation when we use it to estimate the regret gradient. Thus, stochastic gradient descent on $G_n$ corresponds to an off-policy REINFORCE-style update, while the generalized posterior is the Gibbs distribution associated to the same empirical objective.

Beyond this connection at the level of objectives, it is an open challenge to characterize the precise relationship between concentration of the generalized posterior and SGD dynamics. This represents an important next step for future theoretical work (see \cref{section:future_work}). However, for present purposes, we can immediately analyze the learning dynamics of the idealized case, treat this as a source of \emph{empirical} predictions about SGD training, and then validate these predictions with deep RL experiments. We now turn to this analysis.

\section{The Free Energy Formula}\label{section:theory}

In this section, we outline our generalization of Watanabe's free energy formula and Bayesian information criterion for singular statistical models \citetext{\citealp[\S6.3]{WatanabeGreen}; \citealp[Thm.~4]{WatanabeWBIC}}. In particular, \cref{main_theorem_main_text} extends these two results to the generalized inference setting developed in \cref{section:setting}.

With the above definition of the generalized posterior $p(w|D_n)$ the study of the Bayesian learning process in RL reduces to the following mathematical question: where does the generalized posterior concentrate? This is captured by $\mu^\beta_n(U)$ which depends on the random variable $D_n$, and we give precise answers in terms of the asymptotic behavior of the free energy $F_{n,\beta}(U)$. More precisely, we prove that the large $n$ behavior of the generalized posterior is controlled by the singular geometry of the set $W_0 \sub W$ of optima.

\begin{assumption} \label{key_assumption_main_text}
Our main theorem will require imposing the following condition on a Markov decision problem.  Assume that $g(\tau) \ge 0$ for every reachable trajectory $\tau$, i.e. a trajectory that occurs with positive probability under some policy.
\end{assumption}

This is true, for example, if the transition function is deterministic and the model contains an optimal policy (see Proposition \ref{deterministic_positive_prop}).  Next we introduce the analogue of Watanabe's relatively finite variance condition \citep{WatanabeGreen} in our non identically distributed setting.

\begin{assumption} \label{rfv_assumption_main_text}
Let $w_0$ be a local minimum of $G$, with value $G_0$ on an open neighborhood $U$.  We say that $w_0$ satisfies the \emph{relative finite variance} condition if there is a constant $C$ such that, for all $w \in U$ and all parameters $w'$,
\[\bb E_{\tau \sim q_{w'}}\left[\left(\frac{(q_w(\tau) - q_{w_0}(\tau))\,g(\tau)}{q_{w'}(\tau)}\right)^2\right] \le C\,(G(w) - G_0).\]
\end{assumption}

We will show in Lemma \ref{rfv_from_positivity} that Assumption \ref{key_assumption_main_text} implies Assumption \ref{rfv_assumption_main_text} for all global minima.

\begin{theorem} \label{main_theorem_main_text}
 Consider a Markov decision problem satisfying \cref{key_assumption_main_text}.  Then the generalized posterior obeys the following conditions:
 \begin{enumerate}
  \item \emph{Asymptotics of the posterior}: Let $w_0$ be a local minimum of $G$ satisfying Assumption \ref{rfv_assumption_main_text} (for instance any global minimum) and let $U \sub W$ be a sufficiently small open neighborhood of $w_0$.  The generalized tempered posterior has asymptotic behavior
  \[Z_{n,\beta}(U) = \int_U \exp(-n\beta G_n(w)) \phi(w) \d w \sim e^{-n\beta G_n(w_0)}\, n^{-\lambda(U)}\log n^{m(U)-1}\]
  where $\lambda(U), m(U)$ are the learning coefficient and multiplicity of $G - G(w_0)$ on the set $U$ \footnote{The quantities $\lambda(U)$ and $m(U)$ are referred to respectively by algebraic geometers as the real log canonical threshold and the real log canonical multiplicity of the function $G$ restricted to the set $U$.}.
  \item \emph{Expectation of the total loss}: Let $w_0$ be a local minimum of $G$ satisfying Assumption \ref{rfv_assumption_main_text}. Let $\bb E^\beta_n$ represent the expected value with respect to the tempered posterior $\mu_n^\beta$.  Let $\beta$ be a positive function on the natural numbers such that $\beta(n)$ converges to zero as $n \to \infty$ and $n\beta(n) \to \infty$, for instance $\beta(n) = \beta_0/\log n$ (if $w_0$ is a global minimum, $\beta(n)$ may instead converge to any limit, provided $n\beta(n) \to \infty$).  Then there exists an open neighborhood $U'$ of $w_0$ so that for all subneighborhoods $U \sub U'$
  \[\bb E^\beta_n(nG_n(w)|_U) = nG_n(w_0) + \frac{\lambda(U)}{\beta} + o_{\mr P}(\beta^{-1}).\]
 \end{enumerate}
\end{theorem}

We defer proof to \Cref{theory_appendix}, specifically \ref{free_energy_section} and \ref{WBIC_section}. There, we offer a proof following and extending the arguments from \citet{WatanabeGrey, WatanabeWBIC, WatanabeGreen}. There are many details, but the main extensions are in the application of Lyapounov's central limit theorem to address the use of independently but non-identically distributed data, and the identification of \cref{key_assumption_main_text} as the condition under which good local coordinates lift from $G$ to $G_n$ at an optimal parameter.  At a local minimum that is not global we use \cref{rfv_assumption_main_text} in its place: the analogue of the \emph{relative finite variance} condition of \citet{WatanabeGreen}.  \Cref{key_assumption_main_text} implies it at every global minimum (Lemma \ref{rfv_from_positivity}) but not in general.

In fact, the treatment in \cref{theory_appendix} is more general than the statement above in several ways. First, we relax the finite Markov decision problem restriction. Second, we work in a unified setting that simultaneously encompasses both reinforcement learning and supervised learning. Third, we identify a more general fundamental condition from which we recover \cref{key_assumption_main_text} in the RL setting.

\section{Implications of the Free Energy Formula}
\label{section:predictions}

In this section, we leverage \cref{main_theorem_main_text} to derive qualitative predictions about deep reinforcement learning. We first elaborate on the local learning coefficient (\cref{LLC_sec}) and the free energy formula (\cref{section:fef}) and how they combine to describe stagewise development in the generalized posterior (\cref{section:stagewise_dev}).
Then, in \cref{section:deep-predictions}, we interpret this theory to form empirically testable predictions about how deep reinforcement learning should proceed stagewise from simple, high-regret policies to complex, low-regret policies, with policy complexity measured by the LLC.

\subsection{The Local Learning Coefficient}
\label{LLC_sec}

In the Bayesian setting, the LLC provides a rigorous formalization of the notion of the \emph{complexity} of a particular parameter.
Given a point $w^*$ in parameter space with open neighborhood $U$, according to \cref{main_theorem_main_text} we may identify $\lambda(U)$ to leading order in $n$ with the following expression:
\[
     \lambda(U)
     \sim
     n \beta \left[ \bb E_{n,U}^\beta [G_{n}(w)] - G_{n}(\wstar) \right],
\]
where $\bb E_{n,U}^\beta$ is the expectation with respect to the
tempered posterior restricted to the open set $U$.  We may, if we like, obtain the \emph{local} learning coefficient at the point $w^*$ by shrinking the neighborhood $U$, for instance if $W$ is a subspace of $\RR^d$  then we can take a limit over balls of diminishing radius 
\[\lambda(w^*) = \lim_{\eps \to 0} \lambda(B_\eps(w^*)).\] 
More geometrically one can equivalently realize the local learning coefficient in terms of the asymptotic growth rate of basin volumes around a local minimum \citep[see][\S 3.1]{LLC}.

\subsection{The Generalized Free Energy Formula}
\label{section:fef}

\Cref{main_theorem_main_text} explains how the local learning coefficient controls model selection during the learning process. Taking $\beta = 1/\log n$ and writing $F_n = F_{n,1}$, Parts (1) and (2) of \cref{main_theorem_main_text} together imply (see \cref{remark:fef}) the \emph{free energy formula} for a local minimum $w_0$ of $G$ satisfying Assumption \ref{rfv_assumption_main_text}
\begin{equation}\label{eq:free_energy_formula}
F_n(U) = n G_n(w_0) + \lambda(U) \log n + o_{\mr P}(\log n)
\end{equation}
where $U$ is a sufficiently small open neighborhood of $w_0$. This formula describes the concentration of the posterior in a region $U$ in terms of two invariants: the empirical regret $G_n(w_0)$ (lower regret means lower $F_n(U)$, hence higher $Z_n(U)$ and thus posterior concentration) and the learning coefficient $\lambda$ (lower $\lambda$ means a simpler parametrized policy, and higher posterior concentration). Note that the prior $\phi(w)$ contributes to the free energy, and thus posterior concentration, only through constant order terms. These terms can certainly play a significant role for low $n$, but we do not consider this explicitly here.

\subsection{Stagewise Development}
\label{section:stagewise_dev}

The central insight for Bayesian deep RL that we draw from the free energy formula \eqref{eq:free_energy_formula} is that learning consists of an evolving tradeoff between \emph{regret} and \emph{complexity}. This is the translation to RL of a perspective due to \citet[\S 7.6; see also \citealp{Balasubramanian,Chen+2023}]{WatanabeGrey}.

Let $w_1, w_2$ be two local minima of the regret $G$ satisfying Assumption \ref{rfv_assumption_main_text} with respective local neighborhoods $U_1, U_2$. To compare the concentration of the posterior in these two regions we compute the likelihood ratio (setting $\beta = 1$)
\[
\frac{\mu_n(U_1)}{\mu_n(U_2)} = \frac{Z_n(U_1)}{Z_n(U_2)}
\]
which is greater or less than one according to the sign of negative logarithm (which is negative if $U_1$ is more preferred than $U_2$), which has asymptotic expansion
\begin{align*}
F_n(U_1) - F_n(U_2) &= (nG_n(w_1) + \lambda(U_1)\log n  - (nG_n(w_2) + \lambda(U_2) \log n + o_{\mr{P}}(\log n) \\
  &= n(G_n(w_1) - G_n(w_2)) + (\lambda(U_1) - \lambda(U_2))\log(n) + o_{\mr{P}}(\log n)\\
  &= n \delta G - \delta \lambda \log n + o_{\mr{P}}(\log n) = \Big( \frac{n}{\log n} \delta G - \delta \lambda + o_{\mr{P}}(1) \Big) \log n\,.
\end{align*}
where $\delta G = G_n(w_1) - G_n(w_2)$ and $\delta \lambda = \lambda(U_2) - \lambda(U_1)$. From this we infer the following:
\begin{itemize}
    \item \textbf{Observation 1}: if $\delta G > 0$ then for $n$ large the posterior prefers $U_2$ (the lower regret solution).
    \item \textbf{Observation 2}: if $\delta G > 0$ and $\delta \lambda > 0$ then the posterior may prefer $U_1$ for small $n$ (a higher regret, but simpler, parametrized policy can outcompete a lower regret policy).
    \item \textbf{Observation 3}: if $\delta G > 0$ and $\delta \lambda > 0$ then there exists a \emph{critical data set size} $n^*$ such that the posterior prefers $U_1$ for $n < n^*$ and $U_2$ for $n > n^*$. This $n^*$ is the solution\footnote{We are ignoring the $o_{\mr{P}}(1)$ term in this presentation of phase transitions. Whether or not this is reasonable is hard to address theoretically, and we treat it as an empirical question. It is possible to eliminate a Bayesian phase transition, or delay it, by changing the prior in $U_1, U_2$.} of
    \[
    \frac{n^*}{\log n^*} = \frac{\delta \lambda}{\delta G}\,.
    \]
\end{itemize}
Assumption \ref{rfv_assumption_main_text} is where these observations are conditional.  At a global minimum it follows from Assumption \ref{key_assumption_main_text}; at a degenerate local minimum that is not global it must be verified.  The deterministic analogue of the free energy formula, with $G$ in place of $G_n$, holds at every local minimum without this condition and is what the estimator of Section \ref{section:llc_estimation_main} targets.

Note that since the relationship between the free energy and the posterior is exponential, these switches in the sign of the free energy as a function of $n$ can translate to rapid changes in posterior concentration. Following \citet{WatanabeGreen} and \citet{Chen+2023}, we refer to these sudden shifts in the concentration of the posterior as \emph{Bayesian phase transitions}.

This suggests a picture of the Bayesian learning process as a series of phase transitions, where the posterior jumps between neighborhoods of local minima $w_i$ of $G$, according to a changing tradeoff between $G$ and $\lambda$ in the free energy formula. The most typical transition decreases $G$ and increases $\lambda$ (the posterior is ``willing,'' at some data set size, to pay the free-energetic cost of higher complexity in exchange for lower regret).

% This is not that surprising if we think of learning as the process of incorporating information from observed trajectories into the network parameter, but framed differently it has a counter-intuitive implication, flagged already in the introduction: given any finite number of observed trajectories, a \emph{more optimal policy} (lower regret) may not be \emph{more optimal from a Bayesian perspective} (lower free energy) according to the concentration of the generalized posterior. This has implications for AI alignment, which we review in \cref{section:relevance_alignment}.

\subsection{Predictions for Deep Reinforcement Learning}
\label{section:deep-predictions}
\label{section:bayes_vs_sgd}

Thus far we have given a formal description of the dynamics of the generalized Bayesian learning process in RL as a sequence of \emph{Bayesian phase transitions} in which the generalized posterior $p(w|D_n) \propto \exp(-nG_n(w)) \phi(w)$ shifts its concentration between local minima of the population regret $G$. Here, $n$ denotes the number of trajectories ``seen'' by the learner and incorporated into the posterior update.

In SGD-based deep RL, the training process draws \emph{for each gradient step} batches of $B$ trajectories. In a policy gradient algorithm like REINFORCE, the gradient estimator is essentially $\nabla_w G_n(w)$ with $n = B$. After $S$ steps of gradient descent, information from $B \cdot S$ trajectories have influenced the current parameter.

To attempt to line these two learning processes up as much as possible, take the sequence of behavioral parameters to be $w_1, \ldots, w_1, \ldots, w_S, \ldots, w_S$ where $w_i$ is the $i$th step of gradient descent and each parameter is repeated $B$ times (see \cref{section:off_policy_learning}). Then, take the $\tau_i$ to be the contents of mini-batches sampled during SGD. $D_n$ is then the set of trajectories seen so far by SGD.

Note that the precise relationship between $p(w|D_n)$ and, e.g., the distribution of SGD endpoints (with initialization and sampling providing the randomness) is yet to be formally characterized (see \cref{section:future_work}). However, we can take the theoretical phenomenon of Bayesian phase transitions and derive qualitative predictions about the dynamics of SGD-based RL, which can then be empirically tested:
\begin{enumerate}
    \item We predict the existence of \emph{dynamical transitions} in SGD-based deep RL, such that the optimization trajectory moves from plateau to plateau via sudden drops in the regret.
\end{enumerate}
Stagewise development has been observed in deep RL before \citetext{see, e.g., \citealp{Clift2020Logic, mcgrathzero}; also in neuroscience, e.g., \citealp{StagewiseMice}}. Our theory leads us to a more specific prediction:
\begin{enumerate}
    \setcounter{enumi}{1}
    \item The LLC should rise during these dynamical transitions, corresponding to an intuitive increase in policy complexity.
\end{enumerate}
In summary, we predict that SGD-based deep RL should proceed stagewise from simple, high-regret policies to complex, low-regret policies, with policy complexity measured by the LLC.

These predictions parallel the situation in supervised learning. In particular, \citet{Chen+2023} propose the \emph{Bayesian antecedent hypothesis,} which claims that dynamical transitions are the ``shadow'' of Bayesian phase transitions in the following sense: each plateau corresponds to a region of parameter space around some local minima of the regret (a ``phase'') and if there is a dynamical transition between two plateaus, there should be a corresponding Bayesian phase transition (the Bayesian antecedent) in which the posterior shifts concentration between the corresponding regions.
        
Evidence for dynamical transitions in which the LLC increases have been presented by \citet{Chen+2023} in the case of small auto-encoders, and by \citet{Hoogland+2025} and \citet{TransientRidge} for transformers trained on language data or synthetic in-context linear regression data.
 
\section{Empirical Results}\label{section:empirical_results} 

Building upon the theoretical foundations, we wish to demonstrate that estimators for 
the local learning coefficient can be used to provide an 
effective complexity measure that tracks stagewise development of reinforcement 
learning models.  We will demonstrate this in a toy model that nevertheless 
exhibits phase transitions during training and with respect to the 
variation of hyperparameters.

\subsection{Summary} \label{sec:env_summary}
We study a simple gridworld reinforcement learning problem in which an agent is incentivized to find a shortest path between two locations in a grid, and study the training dynamics of a deep neural network parameterizing this problem under policy gradient optimization.  There are two key hyperparameters that we vary.
\begin{itemize}
 \item $\gamma \in [0,1]$, the \emph{discount rate}.
 \item $\alpha \in [0,1]$, the mixing parameter for our initial state distribution.  The initial state distribution is set to $\alpha \Lambda_{\mr{uniform}} + (1-\alpha)\Lambda_{\mr{corner}}$ where in $\Lambda_{\mr{uniform}}$ the goal is uniformly distributed over all possible squares in the grid and in $\Lambda_{\mr{corner}}$ it is always in the top-left corner.
\end{itemize}

During training we see stagewise development in policy space: the model transitions between plateaus in which the regret is roughly constant, the policy remains in a fixed region with qualitatively consistent behavior and the transitions between plateaus are relatively rapid.  We refer to periods where the regret is approximately constant as \emph{phases} and transitions between them as \emph{phase transitions} (and one might compare this phenomenon to the phase transition seen in a toy model of supervised learning in \citealp{Chen+2023}).  For example, we will discuss some of the following common phases in policy space below.
\begin{itemize}
 \item \emph{Phase 1}: in all cases move up with probability $0.5$ and left with probability $0.5$.
 \item \emph{Phase 2a}: move in a deterministic direction either up or left towards the goal, independent of the goal location.
 \item \emph{Phase 2b}: move in a deterministic direction either up or left in a way that always moves towards to goal if possible.
 \item \emph{Phase 3}: move deterministically towards the goal location.
\end{itemize}
We characterize a model as behaving according to one of these phases automatically---each phase describes a subspace of policy space and we characterize a phase as present when the $L^2$ distance from this region is less than a cutoff value ($0.15$ times the maximum distance in the policy polytope.)

\begin{remark}
 As $\gamma$ becomes smaller transitioning away from phase 1 becomes increasingly incentivized, since a smaller $\gamma$ leads to $\gamma^n$ much more rapidly
 decaying to zero as a function of $n$. This means that paths that are much shorter dominate.
 
 As $\alpha$ becomes larger, transitioning away from phases 1 and 2a is increasingly incentivized since it is much more likely that the goal location will be distributed uniformly rather than only found in the top-left corner. So too for transitioning from phase 2b to phase 3.
\end{remark}

In our experiments we estimate how the value of the LLC (which acts 
as our measure of complexity) changes during training and how it is affected by variation in hyperparameters $\alpha$ and $\gamma$.  We observe the following and refer to \cref{phase_transition_section} for more details.
\begin{enumerate}
 \item Phase transitions are accompanied by rapid increases in the LLC estimator, and later phases correspond to larger LLCs.
 \item We see gradual decrease in the LLC estimate within phase 3 in every run---the policy, and therefore the regret, does not change significantly but the representation of the policy within parameter space does change.  The LLC of phase 1 exhibits a consistent value of $28.4\pm5.7$ during training across the 25 independent seeds.
\end{enumerate}

\subsection{Environment and Agent}\label{section:environment_model_training}

In this section we introduce the set-up for our experimental results: the simple environment that determines our Markov decision problem and the neural network model whose behavior we optimize.  

\subsubsection{Environment}

We make use of a simplified version of the \emph{Cheese in the Corner} (CITC) environment described in \citet{Misgen}. The environment is a grid of size $13 \times 13$ with a border of walls around the edge (leaving a $11 \times 11$ grid of empty cells that can be
navigated). One of the central $11 \times 11$ cells contains the goal location (the \emph{cheese}),
and another contains the initial location for the agent (the \emph{mouse}). All of the other cells are empty.
The mouse can be moved in any one of the four cardinal
directions, and if the move would cause the mouse to collide with a wall, the location of the mouse is unchanged.
The objective is to move the mouse to the cell containing the cheese.

The episode ends either when the mouse moves into the same cell as the cheese,
and receives $+1$ return for doing so, or when the environment times
out after $T_{\mr{max}}$ timesteps. On all other timesteps the return is $0$.
The agent is incentivized to maximize the standard expected discounted return, given by
$
\mathbb{E}_{\tau \sim q_w} \left[ \sum_{t=1}^{\infty} \gamma^{t-1} r(s_t) \right]
 = \mathbb{E}_{\tau \sim q_w} \left[ r(s_1) + \gamma r(s_2) + \ldots \right]
$
where $\gamma$ is the discount factor and $r(s_t)$ is the return associated to the state at timestep $t$.
Therefore for a trajectory that reaches the cheese after $T$ steps the return received by the agent
is $\gamma^{T-1}$ if $T \leq T_{\mr{max}}$ and $0$ otherwise.

The state seen by the agent is a tensor of shape $(13 \times 13 \times 3)$,
where the third dimension is a one-hot encoding of the presence 
of a wall ($[1, 0, 0]$), the mouse ($[0, 1, 0]$), and the cheese ($[0, 0, 1]$).
An empty cell is encoded as $[0, 0, 0]$. There are $(11 \times 11) \times (11 \times 11 - 1) = 14,520$ possible states the CITC environment can be in (though given that the cheese cannot move, the state
space partitions into $11^2$ possible connected subsets). We vary the distribution over the choice of cheese location to see the effects on the policies learned by the model during training.

We consider the initial state distributions:
\begin{itemize}
\item $\Lambda_\text{corner}$: The goal always spawns in the top left corner and the starting location is distributed uniformly over all other cells.
\item $\Lambda_\text{uniform}$: The goal and the starting location are distributed uniformly over all distinct pairs of cells.
\item $\Lambda_\alpha$: The initial state is sampled
from $\Lambda_\text{corner}$ with probability $1-\alpha$ and from $\Lambda_\text{uniform}$ with probability $\alpha$. We will write
$\Lambda_\alpha = (1-\alpha) \Lambda_\text{corner} + \alpha \Lambda_\text{uniform}$
to indicate that $\Lambda_\alpha$ is a mixture distribution.
\end{itemize}
We call $\alpha$ the \emph{mixing parameter} for the initial state distribution.

\subsubsection{Model}

The architecture of the model directly follows that used
for \citet{Misgen}, which itself is a modified version of 
the \emph{Large} architecture described in \citet{IMPALA}: 
A convolutional network with 15 convolutional layers, but the 
LSTM that normally follows is replaced 
by a simple feedforward network that takes the output from the convolutional block.  Full details of the architecture are given in \citet{IMPALA}.

\begin{figure}[t]
\centering
\includegraphics[width=0.22\textwidth]{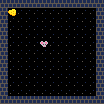}
\hspace{0.3cm}
\includegraphics[width=0.22\textwidth]{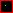}
\hspace{0.3cm}
\includegraphics[width=0.22\textwidth]{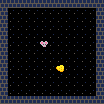}
\hspace{0.3cm}
\includegraphics[width=0.22\textwidth]{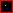}
\caption{Two possible states the environment can be in, together with the observation
seen by the agent.}
\label{fig:cheese_envs}
\end{figure}

The model itself is trained using standard vanilla REINFORCE with no baseline.
A batch of levels is sampled from the environmental distribution $\Lambda_\alpha$,
for which a trajectory for each is generated via interaction with the model.
Actions are sampled on-distribution from the output of the model.
The policy is then updated from the sampled trajectories using the gradient estimator
\begin{equation}\label{eq:training_gradient}
\frac{1}{B} \sum_{i=1}^B \sum_{t=1}^{T} r(\tau_{i,\ge t}) \nabla_w \log \pi_w(a_{i,t} | s_{i,t})
\end{equation}
where $r(\tau_{i,\ge t}) = \sum_{j=t}^{T} \gamma^{j-t} r(s_{i,j}, a_{i,j-1})$ is the \emph{reward-to-go} at time $t$ associated to the $i^{\text{th}}$ trajectory $\tau_i = (s_{i,1},a_{i,1}, \ldots, s_{i,T})$.
We vary both the environmental distribution parameter $\alpha$ as well as the discount
rate $\gamma$ to study the effects on the policies learned by the model during training. The batch size $B$ and trajectory length $T$ are given in \cref{table:hyperparams}.

Given the simplicity of the environment and given that there are only $121 \times 120 = 14520$ states we can analytically compute the 
optimal reward $R_{\mr{max}}$ and the exact reward $R(w)$ 
 of a given current policy $\pi_w$. This allows us to analytically compute the regret,
 $G(w) = R_{\mr{max}} - R(w)$ during various stages of training, as well
 as visualize the policy for each possible location of the agent given a goal location.

\subsection{Details of LLC Estimation}\label{section:llc_estimation_main}

 We define an \emph{estimator} for the local learning coefficient (see \cref{LLC_sec}) using the preconditioned stochastic gradient Langevin dynamics (pSGLD) algorithm~\citep{li2015preconditionedstochasticgradientlangevin} to estimate sampling from the Gibbs posterior. This combines RMSNorm-style adaptive step sizes with SGLD~\citep{WellingTeh}; the update displayed in equation \eqref{SGLD_step} of \cref{appendix:llc_estimation_appendix} is the unpreconditioned form.  Rather than imposing the sharp cutoff associated to the restriction to a neighborhood we use a Gaussian prior centered at $w^*$ with variance $\sigma^2$.  The Gibbs posterior then takes the form
\begin{equation}\label{eq:tempered_posterior_defn_llc}
    p(w|\wstar, \beta, \sigma^2)
    \propto
    \exp \left\{
         -n \beta G_{n,\alpha}(w) - \frac{1}{2\sigma^2} \|w-\wstar\|^2_2
    \right\}\,.
\end{equation}
The hyperparameters are the sample size $n$, the inverse temperature $\beta$, which controls the contribution of the regret, and the variance $\sigma^2$, which controls proximity to
$\wstar$. For more details on these hyperparameters, see \citet{WatanabeWBIC,LLC,Hoogland+2025}.  In \textit{loc.\ cit.\ }the inverse of the variance $\sigma^2$ is referred to as the \emph{localization strength}. Details of the hyperparameters used for LLC estimation can be found in \cref{appendix:llc_estimation}.

\begin{remark}[Off-distribution LLC estimation] \label{off_dist_rmk}
 The interpretation of the LLC estimator as an approximation to the true LLC makes sense near any local minimum $w^* \in W$ of the loss function $G$.  We observe that even though our loss (the regret function) depends on the hyperparameters $\alpha$ and $\gamma$, as long as $\alpha,\gamma \in (0,1)$ the subspace of global minima, and the subspaces associated to the intermediate phases discussed below, will be independent of the hyperparameters.  Therefore it makes sense to compute the LLC estimator of a trained model with respect to \emph{different} values of $\alpha,\gamma$ to those that were used during training.  In particular we may make direct comparisons of the LLC estimates of models trained with different hyperparameter values by estimating LLC using a fixed loss function.
\end{remark}

\subsection{Phase Transitions} \label{phase_transition_section}

An \emph{optimal policy} $\pi^*$ for this environment is any policy where the agent moves closer to the goal with probability one. Any policy that performs
a different strategy must be suboptimal for all mixing parameters
$\alpha$ such that $0 < \alpha \leq 1$ and for all 
discount rates $0 < \gamma < 1$.\footnote{If $\gamma=0$ then 
the behavior of the policy more than one square away from the goal is irrelevant, and for $\alpha=0$ the policy only need to be optimal if the goal is in the top left corner.}
We note that optimal policies are not unique for this environment, 
as there are in general many possible
paths from the agent to the goal.

We are interested in the dynamics of the training process itself and how the learned policy changes over time as a function of the mixing parameter $\alpha$
and the discount rate $\gamma$. We observed that the training process itself exhibits \emph{phase transitions}
in policy space: the policy is stable for a long period of training before abruptly
changing to a new policy as a better strategy is found, rather than continuously moving towards
an optimal policy.

Over many training runs we find that policies can be categorized into one of the following classes:
\begin{itemize}
    \item $\pi^\text{0}$: A uniformly random policy (which is the default at the start of training).
    \item $\pi^\text{1}$: Policy is constant with respect to the state, and moves left with probability $0.5$, or
    up with probability $0.5$.
    \item $\pi^\text{2a}$: Policy is constant with respect to the state, and moves towards
    the top left corner along a shortest path.
    \item $\pi^\text{2b}$: Same as $\pi^\text{2a}$ but choosing a shortest path that passes through the goal on the way to the corner when possible.
    \item $\pi^\text{3} = \pi^*$: An optimal policy that moves the agent along a shortest path to the goal square regardless of where the goal is located.
\end{itemize}

Not all phases are present in every training run, and the dynamics of how policies
transition between phases depend strongly on the choices of $\alpha$ and $\gamma$.
If $\alpha = 0$ then the model has only ever seen environments where the goal square 
is located in the top left corner,
and as shown in \citet{Misgen}, the model learns to \emph{goal misgeneralize},
and will learn to desire only the top left corner, even when the goal square is placed elsewhere.
If $\alpha=1$ then there is no
pressure to choose the top left corner over any other state, and the model will
not learn intermediate phases like $\pi^\text{1}$,$\pi^\text{2a}$,$\pi^\text{2b}$
that privilege the top left corner. The interesting behavior occurs 
for intermediate values of $\alpha$, where the model goes through a sequence of phases, starting out with  $\pi^\text{1}$, then developing to gradually more complex phases such as $\pi^\text{2b}$ and eventually $\pi^\text{3}$.

A training run with $\alpha=0.6,\gamma=0.98$ is shown in \cref{fig:LLC_example}.
We plot the regret, the difference between the value of the current policy,
and that of an optimal policy, as a function of the number of environment steps 
seen during training. Three chosen points during training are illustrated with depictions of the learned policy in each phase.

With the exception of the default policy $\pi^\text{0}$, training plateaus in the intermediate suboptimal policies $\pi^\text{1}$ and $\pi^\text{2b}$ for a long period during training, and the transitions between them are relatively rapid.

\subsubsection{Automatic Phase Detection} \label{subsec:auto_phase_detect}
We will automatically detect when a model lies within a particular phase using an $L^2$ distance in policy space.  Recall that we may identify the space $\Pi$ of policies in the cheese-in-the-corner model as a product of 3-simplices, since a policy must specify a probability distribution on the set $\mc A$ of actions (move up, down, left, right), or a point in the 3-simplex $\Delta(\mc A) = \Delta(\{U,D,L,R\})$, for each state.  In other words
\[\Pi = (\Delta(\mc A)^\mc S) \sub \RR^{4|\mc S|}\]
where the set $\mc S$ of states has $11^4 - 11^2 = 14520$ elements.  Each phase may be identified as follows.
\begin{itemize}
 \item Phase 1 corresponds to a specific point in $\Pi$---the point associated to the probability distribution $\pi(U) = \pi(L) = 1/2$ across all states.  We may compute the $L^2$ distance from this point directly.
 \item Phase 2a corresponds to the intersection of $\Pi$ with the linear subspace where $\pi(R) = \pi(D) = 0$ in all states, $\pi(U) = 0$ in states where the agent is on the top row and $\pi(L) = 0$ in states where the agent is on the left row.  We use the $L^2$ distance from this subspace.
 \item Phase 2b corresponds to the subspace of phase 2a where, in addition, if the agent is directly below the goal $\pi(L) = 0$ and if the agent is directly to the right of the goal $\pi(U) = 0$.  We again use the $L^2$ distance from this subspace.
 \item Phase 3 corresponds to the intersection of $\Pi$ with the linear subspace where $\pi(U) = 0$ when the agents vertical position is greater than or equal to that of the goal, $\pi(D) = 0$ when the agents vertical position is less than or equal to that of the goal,  $\pi(R) = 0$ when the agents horizontal position is greater than or equal to that of the goal and $\pi(L) = 0$ when the agents horizontal position is less than or equal to that of the goal.  We again use the $L^2$ distance from this subspace.
\end{itemize}

Given a phase $P$ let $d_P$ denote the maximum $L^2$-distance from the associated subspace (for instance if $P$ is a $D$-dimensional subspace then $d_P = \sqrt{4|\mc S| - D}$).  We let $\delta = 0.15$ and say that a point $w \in W$ in parameter space is approximately in phase $P$ if
\[\mr{min}_{\pi \in P} \|\pi_w - \pi\|_2 < \delta d_P.\]
The value of $\delta$ is chosen to be smaller than the normalized distance between phase 1 and phase 2b in order to detect the transition between those phases.
The precise choice of $\delta = 0.15$ is not critical for the location of 
the phase transitions, as there is a range over which we could choose $\delta$ with only minor changes
to the corresponding locations of phase transitions (see \Cref{app:delta_range}).

\subsubsection{Distribution of Phase Transitions}\label{section:distribution_phase_transitions}

The training trajectory -- and therefore the time steps at which phase transitions occur -- are random variables, in this section we discuss their empirical distribution.  Specifically we will present statistics in twelve cases: where the model was trained with $\alpha = 0.4, 0.5, 0.6$ and $0.7$ and with $\gamma = 0.97, 0.98$ and $0.99$.  In each hyperparameter case we trained models with these hyperparameters and 14 or 15 independent random seeds.  In Figure \ref{fig:phase_transition_dist} we exhibit the distribution of the training steps at which phase transitions occur as a function of $\gamma$ for each of the four sampled $\alpha$-values.  The models transition into phase 1 consistently early in training, the transition from phase 1 to phase 2b depends on $\gamma$ as we will discuss below.

\begin{figure}[htbp]
\centering
\includesvg[width=\linewidth]{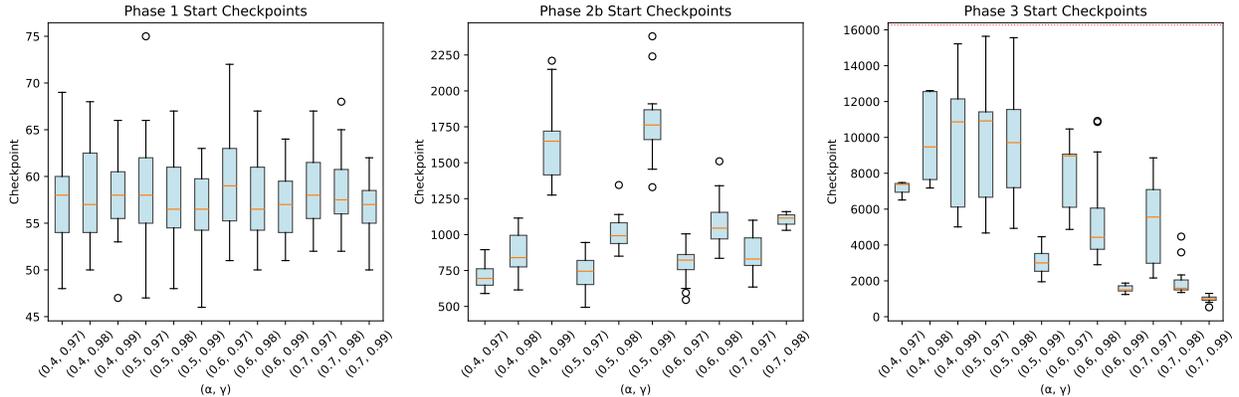}\\[0.4em]
\caption{A series of plots illustrating the model transitions into phases $\pi^1$, $\pi^{2b}$ and $\pi^3$ respectively, as a function of 
mixing parameter $\alpha$ and discount rate $\gamma$. We report the number of checkpoints (policy gradient steps) since training began before we enter each phase. The red line in phase 3 represents when training runs were terminated.} 
\label{fig:phase_transition_dist}
\end{figure}

We may also study the distribution of the observed phase transitions as a function of $\alpha$ and $\gamma$. 
We can note the following observations: 
\begin{enumerate}
    \item The number of training steps before transitioning into $\pi^1$ is very low, and unrelated to the choice of $\alpha$,$\gamma$ presented.
    \item The higher the discount rate, the longer it takes for the model to converge to $\pi^{2b}$, if at all.
    For $\alpha,\gamma = 0.6, 0.99$ and $0.7, 0.99$, the policy region $\pi^{2b}$ is never observed, but is skipped entirely, transitioning
    from $\pi^1 \to \pi^3 $ directly without visiting $\pi^{2b}$ first. For $\alpha,\gamma = 0.7, 0.98$, during training the policy would sometimes, but not always, skip $\pi^{2b}$ and transition directly from $\pi^1 \to \pi^3$.
\end{enumerate}

We can also try to predict the phase transitions based on the hyperparameters with which it was trained, see \cref{appendix:fit_transition_step}.

\subsubsection{LLC Estimates and Phases}
\Cref{fig:LLC_example} shows the LLC estimate and regret across training for a representative model trained with $\alpha = 0.6$ and $\gamma = 0.98$, exhibiting two phase transitions.  We present all the estimated LLC curves in \cref{LLC_plot_appendix}, \cref{fig:llc_grid_alpha_0.5_l2_norms}.  We see rapid increases in the LLC estimate at each phase transition, and the increasingly complex phases are associated with increasingly large LLC estimates.  We also note a gradual decrease of the LLC estimates in the second half of phase 3 as stochastic gradient descent converges to simpler representations of the corresponding local minimum.

Next we discuss the distribution of the LLC estimates within each phase.  We performed LLC estimation for 25 independent models trained with $\alpha = 0.6$ and $\gamma = 0.98$---those hyperparameters being selected to maximize the proportion of runs exhibiting two phase transitions.  We estimated the LLC at six equally-logarithmically-spaced checkpoints in each of the three phases.  \Cref{fig:LLC_statistics} shows the distribution of the LLC estimate at these points during training. Using the internal points within each phase we see mean LLC estimates of 28.42 in phase one (with sample standard deviation $\hat \sigma = 5.65$), 126.17 in phase two ($\hat \sigma = 47.04$) and 328.69 in phase three ($\hat \sigma = 85.21$).

\begin{figure}[htbp]
\centering
\includegraphics[width=0.6\linewidth]{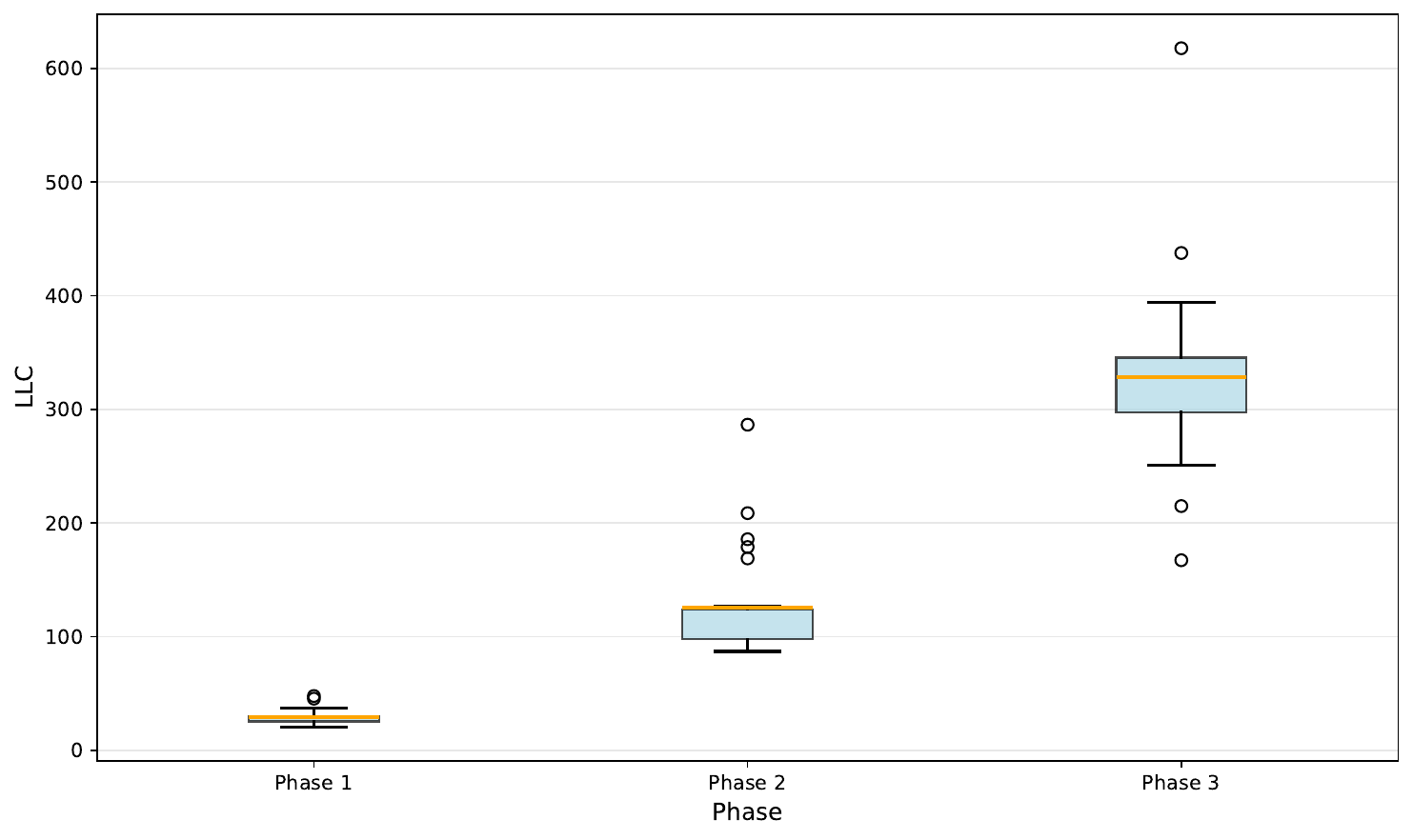}
\caption{Distribution of average LLC estimates across phases. The LLC has been estimated on distribution with $\alpha=0.6$ and $\gamma=0.98$.}
\label{fig:LLC_statistics}
\end{figure}

\section{Discussion}

In this section, we discuss several implications of our results. In \cref{section:relevance_alignment}, we discuss some implication of stagewise learning dynamics for the important problem of AI alignment. In \cref{section:complexity_in_RL}, we survey related work on model complexity and RL. In \cref{section:future_work}, we discuss open questions and new directions raised now that we have extended SLT to RL.

\subsection{Relevance to Alignment}\label{section:relevance_alignment}

Our results have implications for AI alignment. There is a subtle but important distinction between the prevailing (albeit implicit and underdeveloped) view on how policies of deep RL agents develop over training, and the SLT-informed view that this paper provides:

\begin{itemize}
\item \textbf{Standard view}: policies earlier in training are qualitatively similar to policies later in training, which are obtained by \emph{refining} and \emph{adding to} earlier policies. If we haven't converged to an optimal policy it's not for some fundamental reason, it's just a flaw in the optimization process or lack of compute.
\item \textbf{SLT-informed view}: policies earlier in training can be qualitatively different to policies later in training and phase transitions can involve fundamental changes in the underlying algorithm. If we haven't converged to an optimal policy, it may be because it is being screened by a lower complexity solution that is Bayes optimal at the given data set size.
\end{itemize}

By the logic of \cref{section:stagewise_dev} a lower complexity solution can be preferred to a lower regret solution if
\begin{equation}\label{eq:deltalambda_deltaG}
\delta \lambda > \frac{n}{\log n} \delta G\,.
\end{equation}
In cases where there are any theoretical guarantees at all, alignment assumes ``near enough'' convergence to an optimal policy, which is safe \citep{hadfieldmenell2024cooperativeinversereinforcementlearning,kosoy2023learningtheoretic}. Let $\delta G$ be the allowed gap to this optimal safe policy and let us refer to any policy with higher regret as ``misaligned''. We analogize Bayesian learning to SGD training, under some relation between data set size $n$ and the number of trajectories seen over training $B \cdot S$ as in \cref{section:bayes_vs_sgd}, and let $C = n/\log n$ be determined by the available compute. Then by \eqref{eq:deltalambda_deltaG} the generalized posterior can prefer a local minimum of $G$ which is \emph{misaligned but simpler} than the optimal (aligned) policy by more than $C \delta G$ \citep{YAWYE}.

Goal misgeneralization \citep{pmlr-v162-langosco22a,Shah+2022,Misgen} is arguably an example of this phenomena. Indeed \citet[\S 2.2]{pmlr-v162-langosco22a} give inductive bias towards simplicity as their main suggestion for the origin of goal misgeneralization. Our experimental setting is a simplified form of the one studied in that work, and in the training run presented in \cref{fig:LLC_example} if training had stopped at checkpoint 2k, the policy would be ``misaligned'' in the sense that the designer of the environment might have intended a trained agent to seek the cheese (this is after all the optimal policy). 

Similar logic applies to reward hacking in RLHF \citep{lambert2026reinforcementlearninghumanfeedback}. Reward models trained on human preference data may learn simple correlations (e.g., ``longer responses are better'' \citealp{shen2023looselipssinkships} or ``lists are preferred'' \citealp{zhang2025listsemojisformatbias}) rather than complex representations of genuine quality \citep{gao2022scalinglawsrewardmodel}. If the simpler pattern achieves comparable loss on the preference data set, the free energy formula (for supervised learning) predicts it will be favored during training. When this reward model is used to train the policy, these biases are amplified and also interact with the simplicity bias of the RL training process.

There is also a connection to instrumental convergence \citep{omohundro2018basic,bostrom2014superintelligence}. Consider an RL agent trained on a large set of diverse environments. If the environment is easily recognizable, the optimal policy $\pi^*$ will detect the environment and then act in a way tailored to that specific task. However, acquiring resources or control of the environment may be a commonly represented pattern in successful policies across a broad range of environment, and \emph{these behaviors may be simpler to represent} than a large number of tailored policies. Hence a parameter which parametrizes an \emph{instrumentally convergent} policy $\pi^{\text{inst}}$ might be preferred by the posterior even if it is higher regret than $\pi^*$.

\subsection{Model Complexity in Reinforcement Learning}\label{section:complexity_in_RL}

In the (generalized) Bayesian setting, SLT tells us that the LLC provides a rigorous formalization of the notion of complexity of models. Our empirical results are evidence that this notion of complexity is also applicable in the context of stochastic optimization and deep RL. In this section, we survey related work on model complexity and reinforcement learning.

Measurements of complexity, including those quantifying flatness of a local minimum in the loss landscape have been widely studied for neural networks \citep{Hochreiter, Keskar, Neyshabur, DMRB}. Likewise many authors have explored the idea of simplicity bias in neural networks as an implicit regularization affecting generalization \citep{Shah, Teney, Perez, Morwani}. The notion of the loss-complexity tradeoffs and transitions between different solution classes is studied in \citet{Singh, MunnWei, TransientRidge, Wurgaft}.

The direct study of complexity bias in reinforcement learning specifically has been less explored.  The geometry of the reward landscape was explored visually in many examples in \citet{CliffDiving}, which demonstrated the appearance in low-dimensional cross-sections of ``cliffs'' and ``plateaus'' in the landscape.  \citet{Boucher} studied the inductive bias introduced by the technique of maximum entropy learning \citep{MaxEntropy} using a range of complexity measures including those that quantified local flatness of a minimum.  This idea of flatness specifically as a bias favoring more robust optima was investigated in \citet{FlatReward}.  \citet{Boopathy} studied inductive biases in a range of contexts including reinforcement learning using an information-theoretic measure of complexity.  Other authors have applied complexity bias to improve learning performance in reinforcement learning problems \citep{TERL, SimBa}, while others have studied the generalization of RL agents in grid world similar to the environments considered in this paper \citep{RLGeneralization,Misgen}.

The approach taken in this paper of characterizing complexity of a point in parameter space in terms of the local geometry of the regret function is complementary to work that characterizes the complexity of the function class represented by parameter space (including those inspired by Rademacher complexity and Vapnik--Chernovenkis dimension in supervised learning).  This latter approach is taken for reinforcement learning in the complexity measures developed in \citet{BellmanRank, Eluder}.  The \emph{decision-estimation coefficient} defined in \citet{Foster} also follows this approach but with a structure that permits localized estimation near a point in parameter space.

We lastly mention some of the theoretical literature on the dynamics of reinforcement learning.  The dynamics under policy gradient learning was studied in an illustrative model for reinforcement learning in \citet{RLPerceptron}.  The connection between complexity and stagewise development has recently been explored in the supervised learning case, particularly for large language models \citep{WeiEmergent,ChenSudden,ICL1,Hoogland+2025}.

\subsection{Open Questions and Opportunities}
\label{section:future_work}

In this section, we give a discussion of some further directions that the extension of singular learning theory to the reinforcement learning context allows one to pursue.

\subsubsection{Dynamics of Deep Reinforcement Learning}

Understanding the inductive biases of deep RL agents is a fundamental challenge for the science of deep learning.

In this paper, we have bridged the first part of the gap in our current understanding. By extending singular learning theory to an RL setting, we offer the first theory that accounts for the role of geometric degeneracy in reinforcement learning. The qualitative correspondence between stagewise development in our theoretical model and our experimental results is evidence that our theoretical framework has captured something of the fundamental role played by degeneracy in SGD-based deep RL.

This is progress in the right direction, but, as we have discussed, it remains an open question to develop a complete theoretical account of the inductive biases of deep RL proper. In particular, we still need to account for the role played by a local, gradient-based, stochastic optimisation, as compared to the idealized learning process of progressive generalized posterior updates. Our results suggest that a promising path forward lies in studying the geometry of degenerate critical points, such as with the algebraic geometer's toolkit.

\subsubsection{Fundamental Conditions and Stochastic Transitions}

Another important theoretical question is the following. To what extent is it possible to relax the assumption that the transition functions of the Markov problem are deterministic?

We may immediately relax this condition a little to the condition given in \cref{key_assumption}---that no reachable trajectory receives more than the optimal return from its initial state, so that optimal policies obtain it with probability one---but new techniques will be required to go beyond this.  Note that this condition is analogous to a hierarchy of assumptions that appear in Watanabe's work in the context of distribution learning from iid data (essential uniqueness, renormalizability as in \citealp{WatanabeRenormalizable} or the relative finite variance condition in \citealp{WatanabeGreen}).  An approach to generalize beyond this might build on the non-essentially-unique example of \citet{NagayasuWatanabe}, in which a more significant term appears in the free energy.

\subsubsection{Interpretability and Patterning via Weight-Refined LLCs and Susceptibilities}

The extension of SLT to the RL setting also opens the door to some promising practical research directions, based on recent progress applying SLT to develop interpretability and training tools in the supervised learning setting.

\citet{Wang} demonstrated that one can detect the development of structure in parameter space during training by studying so-called \emph{weight-refined} LLCs---LLCs for the restriction of the loss function to a subspace of parameter space.  They study transformer models and detect differentiation of attention head roles during training using this method. Our work shows that these techniques are also applicable in the reinforcement learning context.  By estimating refined LLCs for, for instance, individual layers of the deep neural network in deep reinforcement learning we may detect development of internal structure that is not visible from looking at the policy alone.

The techniques of singular learning theory may be applied to the expectation values of more general observables with respect to the Gibbs posterior.  What's more, if the loss function one considers depends on a hyperparameter (in our setting we saw the dependence of the regret on the hyperparameters $\alpha,\gamma$) one can study the infinitesimal change in these expectation values with respect to the variation of these hyperparameters.  \citet{Susceptibilities} study exactly this setting for supervised learning: their \emph{susceptibilities} correspond to the variation of weight-refined LLCs under infinitesimal perturbations of the data distribution.  In the reinforcement learning setting one can do the same thing, where one varies either the reward function of the Markov problem or the initial state distribution.

One can take this a step further.  \citet{Wang+Murfet2026} introduce a technique, \emph{patterning}, by which one may use susceptibilities to \emph{steer the training process} in an optimal direction for the promotion of desired behaviors, as measured by the maximization or minimization of the posterior expectation value of specified observable quantities.  In other words, one may take advantage of the ability of susceptibilities to measure the dependence of the local geometry of the loss landscape on parameters---such as the mixing parameter in a mixture of data distributions---to modify this geometry in order to promote desirable optima.  In the reinforcement learning context one can use the results we have developed in order to apply the same techniques to modify the favorability of different regions in the space of optima by perturbing structures of the Markov problem such as the initial state distribution. The result would be a method of automatically adapting the RL agent's training distribution so as to promote certain desirable kinds of agent behavior, comparable to unsupervised environment design, as used to elicit robust capability- or goal-generalization \citep[cf.,][]{Dennis+2020,Misgen}.

\section{Conclusion}\label{section:conclusion}

We have extended singular learning theory to deep reinforcement learning, proving that the concentration of the generalized posterior over parameters for policies is governed by the local learning coefficient---an invariant of the geometry of the regret function (\cref{main_theorem_main_text}). This provides a precise theoretical foundation for the intuitive concept of simplicity bias in deep RL: the free energy formula characterizes Bayesian learning as an evolving tradeoff between regret and complexity, predicting that phase transitions should proceed from simple, high-regret policies to complex, low-regret policies.

Our main theoretical insight is that \emph{a more optimal policy is not always more optimal from a Bayesian perspective}. Depending on the number of trajectories seen, a Bayesian learner may prefer a simpler policy with higher regret over a complex policy with lower regret. This is characterized precisely by the critical data set size $n^*$ at which the posterior shifts concentration between competing solutions. One practical implication is that for any finite training process, simpler but less aligned policies may be favored---a phenomenon we connect to goal misgeneralization, reward hacking, and potentially instrumental convergence (\cref{section:relevance_alignment}).

Our empirical results validate the theory's predictions. In a gridworld environment exhibiting stagewise learning, we observe ``opposing staircases'' where phase transitions manifest as rapid decreases in regret accompanied by rapid increases in the LLC.  

The extension of singular learning theory to reinforcement learning opens several directions for future work. First, carefully quantifying how Bayesian posterior concentration relates to SGD dynamics remains an important open problem. Second, the techniques developed here---and related methods for measuring and manipulating model structure \citep{Wang, Susceptibilities, BIF}---should be applicable to alignment-relevant phenomena in more complex reinforcement learning settings including RL training of large language models. Given the centrality of reinforcement learning to the training of frontier AI systems, we believe that understanding and measuring complexity-regret tradeoffs is an important direction for AI safety research.

\section*{Acknowledgements}
We would like to thank Edmund Lau for his involvement in the early stages of this project and for many helpful suggestions.   We are very grateful to Marcus Hutter and Vanessa Kosoy for productive conversations on the topic of this paper during the Mathematical Science of AI Safety focus period at SMRI.  We are also grateful to Rohan Hitchcock for sharing the argument used in the proof of \cref{lemma:rohan} and to Shaowei Lin for providing guidance concerning the control as influence literature.  Finally we would like to thank Zach Furman, Philipp Alexander Kreer and Rumi Salazar for their feedback on an earlier draft.  We thank Edmund Lau for finding an error in the original formulation of Assumption \ref{key_assumption_main_text}.  

This project is funded by the Advanced Research + Invention Agency (ARIA).

\bibliographystyle{abbrvnat}
\bibliography{RL.bib}

\clearpage

\appendix
\crefalias{section}{appendix}
\crefalias{subsection}{appendix}
\crefalias{subsubsection}{appendix}

\section{LLC Estimation}\label{appendix:llc_estimation}

Here we include additional details relevant to \cref{section:llc_estimation_main}.

When running the LLC estimation, we run 5 SGLD regret chains, with 6000 steps per chain and 4800 levels per step. We use a localization strength of 200 (so $\sigma^2 = 1/200$), $n\beta$ of 1000, learning rate of $10^{-6}$. When selecting the SGLD hyperparameters we aim to make the SGLD chains long enough that the mean chain regrets converge while avoiding the following pitfalls.
\begin{itemize}
    \item The model during SGLD sampling spends a significant portion of the time with better performance than at $\wstar$. This would correspond to the SGLD sampler training the model. This happens if the localization is small, $n\beta$ is high and the learning rate is small.
    \item The model during SGLD sampling spends most of the time with regret comparable to that of a generic policy. This pitfall is not present to the same extent in supervised learning since the cross-entropy loss function is unbounded above and gives non-trivial information even in regions of parameter space in which the loss is high.  This pitfall happens when the localization is small, $n\beta$ is small and the learning rate is high.
    \item The model during SGLD sampling stabilizes in a region with loss near $G(\wstar)$, so $G(w) - G(\wstar)$ is small and dominated by noise. This happens when the localization is large and $n\beta$ is small.
\end{itemize}
Since the variance of the mean chain regret is larger in the reinforcement learning setting than in supervised learning we use 5 chains instead of the standard 4 in supervised learning.

\section{Predicting Transitions}\label{appendix:fit_transition_step}

In the context of \cref{section:distribution_phase_transitions} we try to predict transitions based on the hyperparameters $\alpha, \gamma$.

Now, all the values of
$\gamma \in \{0.97, 0.98, 0.99\}$ are relatively close, but using the value of $\gamma$ directly as a feature doesn't 
reflect the hyperbolic nature of the discount factor, as $\gamma \to 1$, the agent becomes ``infinitely'' far-sighted.
So, we instead compute the \emph{effective horizon} $h = \frac{1}{1-\gamma}$, and model the number of gradient
steps taken for the transition $\pi^1 \to \pi^{2b}$ linearly as
$$
c_1 + c_2 \alpha + c_3 h
$$
for learned coefficients $c_1,c_2,c_3$. We find (\cref{fig:phase_analysis}) the best fit to be $-180.2 + 687.9 \alpha + 15.1 h$, with a coefficient
of determination of $R^2 = 0.993$, and a relative error of at worst $\pm 9.1\%$.

\begin{figure}[t]
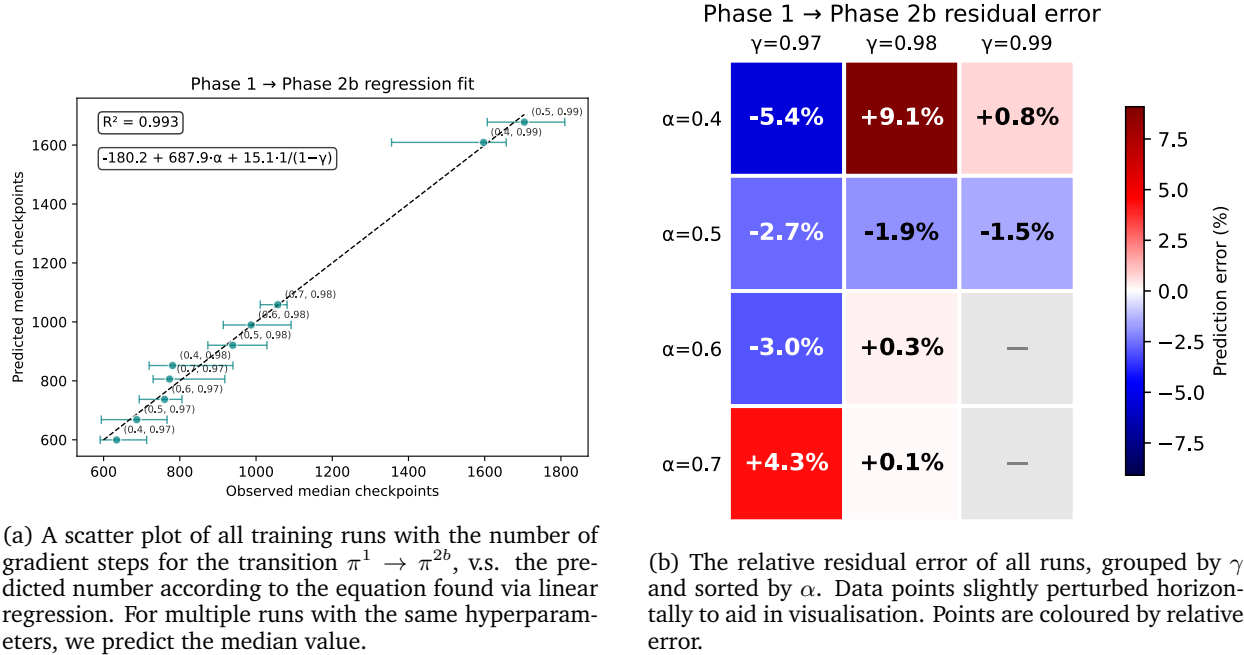

    \centering
    \begin{subfigure}[b]{0.48\textwidth}
        \centering
        \includesvg[width=\textwidth]{img2/phase_regress_fit_Phase_1_to_Phase_2b}
        \caption{A scatter plot of all training runs with the number of 
        gradient steps for the transition $\pi^1 \to \pi^{2b}$, v.s. the predicted number according to
        the equation found via linear regression. 
        For multiple runs with the same hyperparameters, we predict the median value. }
        \label{fig:2d_regress}
    \end{subfigure}
    \hfill
    \begin{subfigure}[b]{0.48\textwidth}
        \centering
        \includesvg[width=\textwidth]{img2/phase_regress_residual_Phase_1_to_Phase_2b}
        \caption{The relative residual error of all runs, grouped by $\gamma$ and sorted by $\alpha$. Data points slightly perturbed horizontally to aid in visualisation. Points are coloured by relative error.}
        \label{fig:phase_regress_error}
    \end{subfigure}
    \caption{Comparison of 2D regression and phase regression error.}
    \label{fig:phase_analysis}
\end{figure}

\section{Model Architecture}

\begin{figure}[htbp!]
\centering
\includesvg[width={0.8\textwidth}]{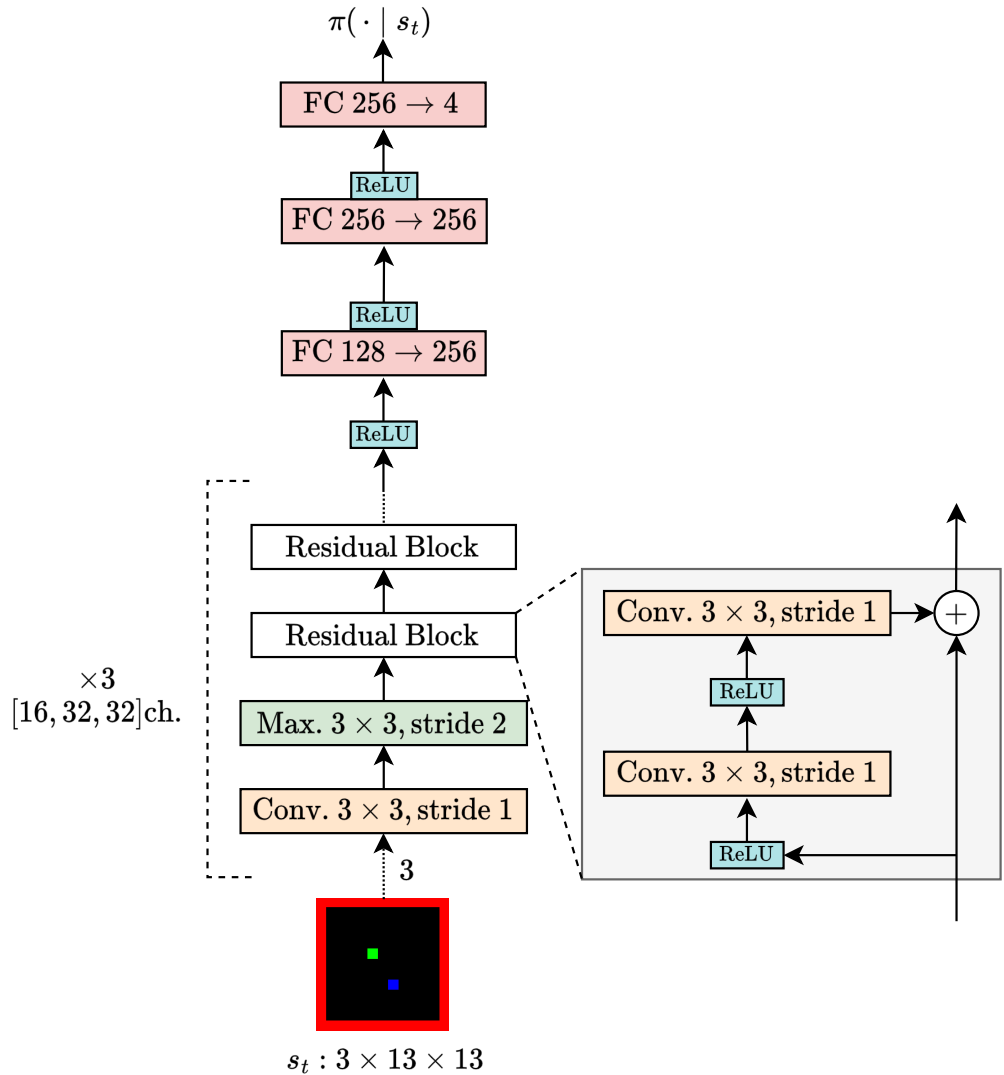}
\caption{Model architecture diagram. Diagram replicated from \citet{IMPALA} with modifications.}
\label{fig:model_architecture}
\end{figure}

The model follows an IMPALA-style convolutional encoder with 
residual blocks presented in \citet{Misgen}, see Figure \ref{fig:model_architecture}.
This then is fed into a small feedforward network (in place of a LSTM)
which then feeds into a policy head, from which we obtain the logits over the actions.

The original network in \citet{Misgen} also adds a value head, which has been
removed as we make no use of it. At a high level the network maps an input
observation \(x \in \mathbb{R}^{C\times H\times W}\)
to policy logits \(\pi(\cdot | s_t) \in \mathbb{R}^{4}\), which can be softmaxed 
to obtain a probability distribution over actions.

\subsection{Convolutional Trunk}
We let $\mathrm{Conv}^{c_{\mathrm{in}}\to c_{\mathrm{out}}}$ denote a $3\times3$ convolution mapping an image with $c_{\mathrm{in}}$ many input channels to an image with $c_{\mathrm{out}}$ output channels. The kernel is always $3\times 3$, with a stride of 1, and padding of 1. We let $\mathrm{MaxPool}$ denote a $3\times3$ max pooling operation with a stride of 2.

The model is then defined as follows.
\begin{align*}
\mathrm{Model} &= \mathrm{Encoder} \to \mathrm{Linear}(256 \to 256) \\
&\quad \to \mathrm{ReLU} \to \mathrm{Linear}(256 \to 4) \\
\mathrm{Encoder} &= \mathrm{Block}(3,16) \to \mathrm{Block}(16,32) \to \mathrm{Block}(32,32) \\
&\quad \to \mathrm{ReLU} \to \mathrm{Flatten} \to \mathrm{Linear}(128 \to 256) \to \mathrm{ReLU} \\
\mathrm{Block}(c_{\mathrm{in}}, c_{\mathrm{out}}) &= \mathrm{Conv}^{c_{\mathrm{in}}\to c_{\mathrm{out}}} \to \mathrm{MaxPool} \to \mathrm{ResBlock}(c_{\mathrm{out}}) \to \mathrm{ResBlock}(c_{\mathrm{out}}) \\
\mathrm{ResBlock}(c)(x) &= x + \mathrm{Conv}^{c\to c}\Big(\mathrm{ReLU}(\mathrm{Conv}^{c\to c}(\mathrm{ReLU}(x)))\Big)
\end{align*}

\subsection{Heads}
One linear head reads from the 256-dimensional embedding, and outputs the logits for the actions.
Outputs preserve the leading batch/time dimensions of the input, which are not notated for brevity.

\subsection{Hyperparameters}

We provide our hyperparameter selections in \cref{hparam_table}.

\begin{table}[h!]
\caption{\label{table:hyperparams} Hyperparameters used for all methods and environments.}
\begin{center}
\begin{tabular}{lr}
\toprule
\textbf{Parameter}          & \textbf{Value} \\
\midrule
\emph{Rollouts}                         \\
\# parallel environments    & 9600       \\
Rollout length              & 64       \\
\# environment steps per gradient step
                            & 614.4k   \\
Discount factor, $\gamma$   & $\in \{$0.97, 0.98, 0.99$\}$ \\

\midrule
\emph{REINFORCE updates} \\
Number of env steps per run & 5B         \\
Number of batches per run   & 8138       \\
Adam learning rate          & 5e-5      \\
Learning rate schedule      & constant  \\

\bottomrule
\end{tabular}
\end{center}
\label{hparam_table}
\end{table}

\subsection{Compute}

We perform each training run on a single NVIDIA RTX-4090 24GB GPU.
The training process runs at about ~450k env steps per second, and each run takes roughly 3 hours. Hyperparameter sweeping for the results used in the paper
was roughly 200 runs, at a cost of roughly 600 RTX-4090 hours. LLC estimation was performed on NVIDIA RTX-5090 32GB GPUs, taking roughly 1 hour per sgld chain, meaning 5 GPU hours per checkpoint. 

\section{Theoretical Foundations of Singular Learning Theory in Reinforcement Learning}
\label{theory_appendix}

In this appendix, we prove the results stated in \cref{section:theory}.
\Cref{MarkovSection} contains preliminary definitions, including a formulation of analytic Markov decision processes that relaxes the finiteness conditions adopted in the main text for brevity.
\Cref{section:slt_gen_bayes} introduces a unified definition of an inference problem that subsumes both the generalized Bayesian inference for RL developed in the main text and also the supervised learning setting of Watanabe, among other settings, and then proves the main results in this unified setting.

\subsection{Markov Problems and the Regret} \label{MarkovSection}
We will begin by setting up the precise setting for Markov decision processes to which our results may be applied.

\subsubsection{Initial Setup}

In this section, by \emph{space} we mean a locally compact Hausdorff space equipped with its Borel $\sigma$-algebra.  Most of the spaces that we consider will be real analytic manifolds (or more generally real analytic manifolds with corners).  If $M$ is a real analytic manifold we will write $C^\omega(M)$ for the vector space of real analytic functions on $M$.  This is a locally convex topological vector space when equipped with the compact-open topology.

Fix a \emph{partially observable Markov decision process} with real analytic manifolds $\mc S$ of states, $\mc A$ of actions and $\mc O$ of observations equipped with volume forms $\d s, \d a, \d o$ respectively.  Suppose that the expectation value of the reward exists for all state action pairs $(s,a)$  and denote it by $r(s,a)$.  Assume that the function $r$ is real analytic.  We fix a natural number $T_{\mr{max}}$ and require that the Markov decision process terminates after $T_{\mr{max}}$ actions.

\begin{example}
 Recall that the Markov decision process is called \emph{finite} if $\mc S, \mc A, \mc O$ are all finite sets equipped with their counting measures.
\end{example}

\begin{remark}
 \citet{ToussaintStorkey} study Markov decision problems of unbounded length by considering an infinite mixture of bounded length decision problems.  We will not pursue this approach in the formalization considered below but we expect that the techniques will generalize to their setting.
\end{remark}

\begin{definition}
A \emph{trajectory} is a finite sequence
\[\tau = (s_0, o_0, a_1, s_1, o_1, \ldots, a_{T_{\mr{max}}}, s_{T_{\mr{max}}}, o_{T_{\mr{max}}})\]
where $s_i \in \mc S, o_i \in \mc O$ and $a_i \in \mc A$.  We write $\mc T$ for the real analytic manifold of all trajectories with canonical volume form $\d \tau$.
\end{definition}

If we fix a discount factor $\gamma \in [0,1]$, the \emph{reward of a trajectory} is given by
\[r(\tau) = \sum_{i=1}^{T_{\mr{max}}} \gamma^i r(s_{i-1},a_i).\]
Note that we reuse the symbol $r$ here.  We will no longer need to refer to the reward of a single state-action pair so there will be no ambiguity.

\begin{definition}
 A \emph{policy} is a measurable function $\pi \colon \mc O \to \Delta(\mc A)$ assigning a probability distribution over actions to each observation.  We equip $\Delta(\mc A)$ with its weak${}^*$ topology and write $\Pi$ for the (non-locally-compact) space of all policies equipped with the compact-open topology.
\end{definition}

Policies determine probability distributions on the space $\mc T$ of trajectories as follows.
\begin{definition}
 Fix a policy $\pi \in \Pi$ and an initial distribution $\Lambda \in \Delta(\mc S)$.  If we write $p(o_i|s_i)$ and $p(s_i|s_{i-1},a_i)$ for the transition probabilities in the Markov process then we may define a probability measure $q_\pi \d \tau$ on $\mc T$ with density function given by
 \[q_{\pi}(\tau) = \Lambda(s_0)p(o_0|s_0)\prod_{i=1}^{T_{\mr{max}}} \pi(a_i|o_{i-1})p(s_i|s_{i-1},a_i)p(o_i|s_i).\]
\end{definition}

\begin{definition}
A \emph{model} consists of a manifold with corners $W$ (the \emph{parameter space} of the model) equipped with a volume form $\d w$ and a continuous function $F \colon W \to \Pi$ (the \emph{parameterization}).  We will sometimes write $\pi_w$ for the image $F(w)$ of a point under the parameterization map and $q_w$ for the density function $q_{F(w)}$.
\end{definition}

In order to generalize Watanabe's singular learning theory arguments we will require a \emph{real analyticity} condition on the functions that appear in our optimization problem.  Let us clarify what this means.
\begin{definition}
 If $X, Y$ are real analytic manifolds (possibly with corners), a function $f \colon X \to \Delta(Y)$ is \emph{real analytic} if for all finite subsets$\{y_1, \ldots, y_m\} \sub Y$ the function $\mr{proj}_y \circ f \colon X \to \Delta^m$ is real analytic, where $\mr{proj}_y$ is the projection onto the $m$-simplex spanned by $\{y_1,\ldots, y_m\}$.
\end{definition}

\begin{assumption} \label{analytic_assumption}
 We say our Markov decision problem and its model are \emph{real analytic} if the reward function $r$, the transition functions of the Markov problem and the parameterization map $F$ are all real analytic.  From now on all Markov decision problems are assumed to be real analytic.
\end{assumption}

\begin{example}
 Finite Markov decision problems are automatically real analytic, and if we parameterize such a problem by a neural network with real analytic activation function (for instance linear or polynomial activations, sigmoids, or the GeLU or swish functions) then the model is also real analytic.
\end{example}

We may now introduce the optimization objective in our reinforcement learning problem, the \emph{regret} function.
\begin{definition} \label{reward_def}
 The \emph{expected reward} at parameter $w \in W$ is
 \[R(w) = \bb E_{\tau \sim q_w}(r(\tau)).\]
 Suppose that $R \colon W \to \RR$ is bounded above and let $R_{\mr{max}} = \sup_{w \in W} R(w)$.  The \emph{regret} at parameter $w \in W$ is
 \begin{align*}
 G(w) &= R_{\mr{max}} - R(w).
 \end{align*}
\end{definition}

\begin{remark}
We will further identify
\[G(w) = \bb E_{\tau \sim q_w}(g(\tau)),\]
where $g(\tau) = V^*(s_0) - r(\tau)$ with $s_0$ the initial state of $\tau$ and $V^*(s_0) = \sup_\pi \bb E(r(\tau) \mid s_0)$ the optimal expected return starting in state $s_0$, the supremum taken over all policies.  Indeed, the baseline term $\bb E_{\tau \sim q_w}(V^*(s_0)) = \bb E_\Lambda(V^*)$ is independent of $w$, so $\bb E_{\tau \sim q_w}(g(\tau)) = G(w) + c$ with $c = \bb E_\Lambda(V^*) - R_{\mr{max}} \ge 0$.  We assume that the model attains the optimal value, that is $c = 0$; this is the condition $\min G = 0$ of Assumption \ref{key_assumption} for the loss $\bb E_{\tau \sim q_w}(g(\tau))$, and under it the two expressions agree.
\end{remark}

\begin{remark}
If we take $W$ to be compact then $R$ is automatically bounded and attains its bounds since it is continuous.  Therefore $G$ is bounded below by zero and attains this lower bound along some non-empty subspace $W_0 \sub W$.
\end{remark}

\subsubsection{Importance Weighting Estimators}
For $n \in \bb N$ let $w_1, \ldots, w_n$ be a sequence of points in $W$.  Let $\tau_1, \ldots, \tau_n$ be a sequence of random variables in $\mc T$ where $\tau_i$ is sampled from the distribution $q_{w_i}$; in particular suppose that $q_{w_i}(\tau_i)>0$ for all $i=1, \ldots, n$.  We will define a statistic valued in continuous functions on $W$ that estimates the regret.

\begin{definition}
Define a random variable, the \emph{importance sampling regret estimator}, as follows.
\begin{align*}
 G_n(w) &= \frac 1n \sum_{i=1}^n \frac{q_w(\tau_i)} {q_{w_i}(\tau_i)} g(\tau_i) \\
 &= \frac 1n \sum_{i=1}^n \prod_{j=1}^{T_{\mr{max}}} \frac{\pi_w(a_{j,i}|o_{j-1,i})}{\pi_{w_i}(a_{j,i}|o_{j-1,i})}  g(\tau_i)
\end{align*}
where $a_{j,i}$ represents action $j$ in trajectory $\tau_i$ and similarly for the observation $o_{j-1,i}$.
\end{definition}

Observe that Assumption \ref{analytic_assumption} guarantees that the random variables $G_n(w)$ are real analytic function valued.  That is, $G_n$ is real analytic as a function
\[W \times W^n \times \mc T^n \to \RR\]
defined away from the subspace on which $q_{w_i}(\tau_i) = 0$, with poles as this subspace is approached.

We will want to refer to the importance weighted regret of a trajectory frequently enough in the subsequent section that it is worth establishing consistent notation for it.
\begin{definition}
For $w,w' \in W$ and $\tau \in \mc T$ such that $q_{w'}(\tau) > 0$ we define
\[f(w,w',\tau) = \frac {q_w(\tau)}{q_{w'}(\tau)}g(\tau).\]
\end{definition}

\subsection{Singular Learning Theory for Generalized Bayesian Inference}\label{section:slt_gen_bayes}

We will introduce at this point an additional assumption.
\begin{assumption}
 Assume that the parameter space $W$ is compact.
\end{assumption}
In particular the set $W_0 = G^{-1}(0)$ of global minima of the regret is non-empty.  Recall that $W$ is a manifold with corners; we use $\interior W$ to denote its interior.

\subsubsection{Setup for Generalized Bayesian Inference}

In this section we will describe a general setting for generalized Bayesian inference that allows us to study singular learning theory for supervised learning and reinforcement learning, among other situations, in a unified fashion.  We will see shortly in Example \ref{MDP_example} that the Markov decision problem setting introduced in the previous section is an example of the situation we will describe now.  Our basic setting will comprise the following.
\begin{itemize}
 \item Let $X$ be a real analytic manifold equipped with a volume form $\d x$.
 \item Let $w \mapsto q_w \d x$ be a real analytic family of signed measures on $X$ parameterized by $w \in W$.  Assume that if $w \in \interior W$ then $q_w(x) \ne 0$ for all $x \in X$.
 \item Let $g \colon X \to \RR$ and $h \colon X \to \RR_{\ge 0}$ be real analytic functions.
\end{itemize}

\medskip

\begin{definition}
 The \emph{loss function} $G \colon W \to \RR$ is defined to be:
 \[G(w) = \int (g(x) q_w(x) - h(x)) \d x.\]
\end{definition}

\begin{assumption} \label{key_assumption}
 Suppose that $\mr{min}(G(w)) = 0$, and call $w \in W$ with $G(w) = 0$ an \emph{optimal parameter}.  Assume that there is an open neighborhood $N$ of the set $W_0 = G^{-1}(0)$ of optimal parameters such that for every $w \in N$ the set
 \[\{x \colon g(x)q_w(x) - h(x) < 0\}\]
 has Lebesgue measure zero. 
\end{assumption}

This assumption implies that for every optimal parameter $w$ the set $\{x \colon g(x)q_w(x) - h(x) \ne 0\}$ has Lebesgue measure zero (see Lemma \ref{g_positive_given_assumption}).

\begin{remark}
 If we're interested in optimization for $G(w)$ observe that the choice of $h$ only changes $G$ by an additive constant, so one can try to choose $h$ in such a way that the assumption is satisfied.
\end{remark}

\begin{definition} \label{importance_sampling_def}
 Given a sequence of points $w_1, \ldots, w_n$ in $\interior W$ and samples $x_i \sim q_{w_i}$ for $i= 1, \ldots, n$, define a random variable, the \emph{importance sampling estimator}, by
 \begin{align} \label{eq:additive}
  G_n(w) &= \frac 1n \sum_{i=1}^n f(w,w_i,x_i) \\
  \text{where } f(w,w_i,x_i) &=  \frac{q_w(x_i)}{q_{w_i}(x_i)}g(x_i) -  \frac 1{q_{w_i}(x_i)}h(x_i). \nonumber
 \end{align}
\end{definition}

\begin{definition} \label{ls_analytic_def}
Let $s \in [1,\infty]$ and let $V \sub \interior W \bs W_0$ be compact.  A function $F \colon X \times W \times V \to \RR$ is \emph{$L^s$-valued real analytic, uniformly on $V$}, if every $(w, w') \in W \times V$ has a chart with local coordinates $(u, u')$ in which
\[F(x, u, u') = \sum_{\alpha, \alpha'} a_{\alpha,\alpha'}(x)\, u^\alpha u'^{\alpha'}\]
with $a_{\alpha,\alpha'} \in L^s(q_v)$ for all $v \in V$, the series converging absolutely in $L^s(q_v)$ for $(u, u')$ in a polydisk of positive radius, uniformly in $v \in V$.  The sum ranges over pairs of multi-indices $\alpha, \alpha' \in \ZZ_{\ge 0}^d$.
\end{definition}

\begin{assumption} \label{ls_analytic_assumption}
There is an even integer $s > 2$ such that, for every compact $V \sub \interior W \bs W_0$, the importance sampling estimator $f(w, w', x)$ is $L^s$-valued real analytic in $(w, w')$, uniformly on $V$.
\end{assumption}

This holds when the sample space $X$ is finite, as it is for the Markov decision problems of Section \ref{MarkovSection}: each $L^s(q_v)$ norm is then bounded by the supremum over $X$, so real analyticity of $f$ in $(w, w')$ for each $x$, which follows from the real analyticity of the family $q_w$ and the nonvanishing of $q_{w'}$ on $\interior W$, suffices.  Partial derivatives in $(w, w')$ of an $L^s$-valued real analytic function are again $L^s$-valued real analytic \citep[Lemma 4.2.7]{HitchcockThesis}; this is used for the bound on $\nabla \psi_n$ in the proof of Lemma \ref{lemma:rohan}.

\begin{example} \label{supervised_example}
Our first example is considered in Watanabe's original work \citep{WatanabeGrey}: supervised learning for a model where the true distribution is realizable.  Choose any compact parameter space $W$ and sample space $X$ and specify a real analytic family $p_w$ of probability distributions on $X$.

  Let $q_w(x) = -\log p_w(x)$.  Fix a point $w^* \in W$ and let $g(x) = p_{w^*}(x)$ and $h(x) = -p_{w^*}(x) \log p_{w^*}(x)$.
  Then
  \[G(w) = \int p_{w^*}(x) \log \frac{p_{w^*}(x)}{p_{w}(x)} \d x = \mr{KL}(p_{w^*}\parallel p_w),\]
  that is, the loss agrees with the Kullback--Leibler divergence from the ``true distribution'' $p_{w^*}$.  A point $w$ is optimal if $p_w = p_{w^*}$.  If so then $g(x)q_w(x) - h(x) = 0$ for all $x$, so the conclusion of Lemma \ref{g_positive_given_assumption} holds automatically; the sign condition of Assumption \ref{key_assumption} does not, since the integrand $p_{w^*}\log(p_{w^*}/p_w)$ is signed.  In this example the standard form of Proposition \ref{omega_zero_prop} is supplied instead by \citet[Theorem 6.1]{WatanabeGrey}, see Remark \ref{rfv_remark_original}. 
  
  For the estimators in this example we generalize Definition \ref{importance_sampling_def} slightly in this case and instead consider $x_i$ samples from $p_{w^*}$ for $i=1, \ldots n$ with
  \[f(w,w^*,x_i) = \frac{q_w(x_i)}{p_{w^*}(x_i)}g(x_i) -  \frac 1{p_{w^*}(x_i)}h(x_i)\]
  so that
  \[G_n(w) = \frac 1n \sum_{i=1}^n f(w,w^*,x_i) =  -\frac 1n \sum_{i=1}^n \log \frac{p_w(x_i)}{p_{w^*}(x_i)}.\]
  The proof of the main results continues to hold for this modified definition of $f$ but we do not track this modification carefully in the text below; in this case our results coincide with those in Watanabe's original work.
\end{example}

\begin{example} \label{MDP_example}
Now let us consider a Markov decision problem as in \ref{MarkovSection} and let $X\sub\mc T$ be the space of trajectories that have positive probability under some policy (for deterministic transition functions, those consistent with the dynamics and with $\Lambda(s_0) > 0$), so that $q_w(\tau) > 0$ on $X$ whenever every action probability of $\pi_w$ is positive, as the setting above requires.  Let $\Pi$ be the space of policies, and choose an analytic parameterization map $F \colon W \to \Pi$ sending $w$ to $\pi_w$.  Let $q_w$ be the distribution on $\mc T$ associated to this choice of policy.  If $\tau \in \mc T$, let $g(\tau)$ be defined as in Definition \ref{reward_def} and let $h(\tau) = 0$.

Assumption \ref{key_assumption} is not necessarily satisfied; it holds if no reachable trajectory receives more than the optimal return from its initial state, in particular if the transition functions are deterministic and the model contains an optimal policy (Proposition \ref{deterministic_positive_prop}).  We will discuss this further in Section \ref{free_energy_section} below.
\end{example}

Note that in the Markov decision problem setting Assumption \ref{key_assumption} holds in particular if the decision problem has deterministic transition functions, meaning that the distributions $p(o|s)$ and $p(s^\prime|s,a)$ are all delta distributions, and the model contains an optimal policy.  Say a trajectory $\tau$ is \emph{reachable} for an MDP if it has positive probability under some policy, of any kind; for deterministic transition functions this holds if and only if $\Lambda(s_0) > 0$ and each transition of $\tau$ is the one the dynamics prescribe.

\begin{prop} \label{deterministic_positive_prop}
Suppose the Markov decision problem has deterministic transition functions.  Then $g(\tau) \ge 0$ for all reachable trajectories $\tau \in \mc T$.  If moreover $\pi_w$ is an optimal policy then $q_w(\tau) = 0$ whenever $g(\tau) > 0$.  In particular, if the model contains an optimal policy, Assumption \ref{key_assumption} holds, with $N = W$.
\end{prop}

\begin{proof}
Under deterministic transition functions a trajectory $\tau$ is determined by its initial state $s_0$ and its sequence of actions $a_1, \ldots, a_{T_{\mr{max}}}$; thus there is a policy producing $\tau$ with probability one, so $r(\tau) \le V^*(s_0)$ and $g(\tau) \ge 0$.  If $\pi_w$ is optimal then $G(w) = \sum_\tau q_w(\tau) g(\tau) = 0$ with every summand non-negative, so $q_w(\tau) = 0$ whenever $g(\tau) > 0$.  The sign condition of Assumption \ref{key_assumption} holds for every $w \in W$ because $g \ge 0$ on the support of every $q_w$, and $\min G = 0$ since the model contains an optimal policy.
\end{proof}

Given the above setup, we will study \emph{generalized Bayesian inference} in the sense of \citet{Zhang, BHW} for the loss function $G$.  To that end, fix an analytic prior probability distribution $\phi(w) \d w$ on $W$, with $\phi > 0$ on $W$, and a real number $\beta > 0$.
\begin{definition}
 The \emph{generalized tempered posterior} or \emph{Gibbs posterior} associated to $\phi, \beta$ and the random variable $G_n$ is the probability distribution
 \[\mu_n^\beta(U) \propto \int_U \exp(-n\beta G_n(w)) \phi(w) \d w.\]
\end{definition}

\begin{remark}
The generalized posterior with $\beta = 1$ is determined canonically by Theorem 1 of \citet{BHW}.  Indeed, given a prior distribution $\phi(w)\d w$ and the loss function $f(w,w',x)$ we consider the collection of all probability measures of the form
\[\mu_{(\psi)} \propto \psi(f(w,w',x), \phi(w))\d w,\]
where $\psi$ is a continuous function of two real variables.  Bissiri, Holmes and Walker's argument shows that the generalized posterior is uniquely characterized within this family by the condition that
\[\psi(f(w,w',x_2), \psi(f(w,w',x_2), \phi(w)))) = \psi(f(w,w',x_1)+f(w,w',x_2), \phi(w)).\]
Their proof applies in the case where $W$ is countable and discrete, but we note that since the function $\psi(f(w,w',x_2), \psi(f(w,w',x_2), \phi(w)))) - \psi(f(w,w',x_1)+f(w,w',x_2), \phi(w))$ is continuous in $w \in W$ it is enough to show that it vanishes on a countable dense subset of $W$.  Together with Remark \ref{importance_weight_canonical_remark} below we see that the generalized tempered posterior is canonically characterized by the loss function under natural assumptions.
\end{remark}

Our goal is to prove theorems about the properties of this posterior distribution, generalizing the arguments of Watanabe.  Our main results are as follows.

\begin{theorem} \label{main_theorem}
 \begin{enumerate}
  \item \emph{Standard form of the estimator}: If $w_0 \in W$ is an optimal parameter then we may find a log resolution $\varpi$ of $G$ locally near $w_0$ and a local coordinate $u$ for which
  \[G_n(\varpi(u)) = u^{2k} - n^{-1/2}u^k\xi_n(u)\]
  where $\xi_n(u)$ converges in distribution to a centered Gaussian process, uniformly on compact subsets of the chart (Corollary \ref{full_chart_convergence}).
  \item \emph{Asymptotics of the posterior}: Let $w_0$ be a local minimum of $G$ satisfying Definition \ref{rfv_def} (for instance an optimal parameter) and $U \sub W$ be a sufficiently small open neighborhood of $w_0$.  The generalized tempered posterior has asymptotic behavior
  \[\int_U \exp(-n\beta G_n(w)) \phi(w) \d w \sim e^{-n\beta G_n(w_0)}\, n^{-\lambda}\log n^{m -1}\]
  where $\lambda, m$ are the log canonical threshold and multiplicity of $G - G(w_0)$ on the set $U$.
\item \emph{Expectation of the total loss}: Let $w_0$ be a local minimum of $G$ satisfying Definition \ref{rfv_def}. Let $\bb E^\beta_n$ represent the expected value with respect to the tempered posterior $\mu_n^\beta$.  Let $\beta$ be a positive function on the natural numbers such that $\beta(n)$ converges to zero as $n \to \infty$ and $n\beta(n) \to \infty$ (if $w_0$ is an optimal parameter, $\beta(n)$ may instead converge to any limit, provided $n\beta(n) \to \infty$).  Then there exists an open neighborhood $U'$ of $w_0$ so that for all subneighborhoods $U \sub U'$
  \[\bb E^\beta_n(nG_n(w)|_U) = nG_n(w_0) + \frac \lambda \beta + o_{\mr P}(\beta^{-1}).\]
 \end{enumerate}
\end{theorem}

\subsubsection{Convergence of Error Terms}
As the first step in our proof of Theorem \ref{main_theorem} we will need to establish a version of the results from \citet[\S 5]{WatanabeGrey} that applies in the generalized Bayesian inference setting in which our samples $x_i$ are not identically distributed.  We will first establish that a version of the central limit theorem my be applied to our importance weighted regret estimators.

\begin{lemma} \label{importance_weighting_mean_lemma}
 The expectation value of the importance weighted loss estimator is $\bb E(G_n(w)) = G(w)$ for all $n$.
\end{lemma}

\begin{proof}
 This is a straightforward computation:
 \begin{align*}
  \bb E(G_n(w)) &= \frac 1n \sum_{i=1}^n \bb E_{x_i \sim q_{w_i}}\left(\frac{q_w(x_i)} {q_{w_i}(x_i)} g(x_i) -  \frac 1{q_{w_i}(x_i)}h(x_i) \right) \\
  &= \frac 1n \sum_{i=1}^n \left(\int (g(x_i) q_w(x_i) - h(x_i)) \d x_i \right) \\
  &= G(w).
 \end{align*}
\end{proof}

\begin{remark} \label{importance_weight_canonical_remark}
This property, together with the \emph{additivity} condition in the statistics $G_n$ given by equation \ref{eq:additive} naturally characterizes the importance weighted loss estimator.  Indeed, consider the case $n=1$ (where we usually denote $G_1$ by $f$).  The condition $\bb E(f(w,w')) = G(w)$ may be rewritten as
\[\int \left(g(x) \frac{q_w(x)}{q_{w'}(x)} - h(x) \frac 1{q_{w'}(x)} \right) - f(w,w',x) \d x = 0.\]
So our definition of the importance weighted estimator $f(w,w')$ is the unique function for which the integrand identically vanishes.
\end{remark}

For the rest of this section we will impose some conditions on our importance weighting estimator.  Suppose $w, w_i \in \interior W \bs W_0$ and $w_i \to w^*$ as $i \to \infty$ for some $w^* \in \interior W \bs W_0$. Suppose that the importance weighted regret $f(w,w',x)$ has finite positive variance $0 <\sigma(w,w')< \infty$ for all $w' \in \interior W \bs W_0$.  Suppose also that there exists $\delta > 0$ so that the $2+\delta$-moment of $f(w,w',x)$ exists for all $w' \in \interior W \bs W_0$.

\begin{prop} \label{psi_CLT}
 The random variables
 \[
  \psi_n(w) = \frac{\sqrt n ( G(w) - G_n(w))}{\sigma(w, w^*)}
 \]
 converge in distribution to $N(0,1)$ as $n \to \infty$.
\end{prop}

\begin{proof}
We have just seen in Lemma \ref{importance_weighting_mean_lemma} all individually have expected value $G(w)$.  To establish the result we will verify Lyapounov's condition, that is, we need to check that, for some $\delta > 0$,
 \[\frac {\sum_{i=1}^n\|f(w,w_i, x_i) - G(w)\|_{2+\delta}^{2+\delta}}{\left(\sum_{i=1}^n\|f(w,w_i, x_i) - G(w)\|_{2}^{2}\right)^{\frac{2+\delta}2}} \to 0\]
 as $n \to \infty$.  Choose $\delta > 0$ so that the relevant moment exists.  Then since $w_i \to w^*$ we may find a compact subspace of $\interior W \bs W_0$ containing $w^*$ and $w_i$ for all $i$, and therefore may find we may choose constants $c_2, c_{2+\delta} > 0$ so that
 \[\|f(w,w_i,x) - G(w)\|_2 \ge c_2, \quad \|f(w,w_i,x) - G(w)\|_{2+\delta} \le c_{2+\delta}\]
 for all $i \ge 1$.  Then
 \begin{align*}
  \frac {\sum_{i=1}^n\|f(w,w_i,x_i) - G(w)\|_{2+\delta}^{2+\delta}}{\left(\sum_{i=1}^n\|f(w,w_i,x_i) - G(w)\|_{2}^2\right)^{\frac {2+\delta}2}} &\le \frac{nc_{2+\delta}^{2+\delta}}{n^{\frac {2+\delta}2}c_2^{2+\delta}} \\
  &= \frac{c_{2+\delta}^{2+\delta}}{c_2^{2+\delta}}n^{-\frac \delta 2} \\
  &\to 0
 \end{align*}
 as $n \to \infty$.  Lyapounov's central limit theorem then says that
 \[\frac 1{\sqrt n} \sum_{i=1}^n \frac{G(w) - f(w,w_i,x_i)}{\sigma(w, w_i)} \to N(0,1)\]
 in distribution as $n \to \infty$.

 To complete the proof we replace $\sigma(w,w_i)$ by $\sigma(w,w^*)$.  Since $w_i \to w^*$ and $\sigma(w,-)$ is continuous and positive on a compact subset $V \sub \interior W \bs W_0$ containing the sequence $w_i$ and the point $w^*$, we have $\sigma(w,w_i) \to \sigma(w,w^*)$.  The difference
 \[D_n = \frac 1{\sqrt n} \sum_{i=1}^n \left(G(w) - f(w,w_i,x_i)\right)\left(\frac 1{\sigma(w,w_i)} - \frac 1{\sigma(w, w^*)} \right)\]
 is a sum of independent centered terms with
 \[\bb E(D_n^2) = \frac 1n \sum_{i=1}^n \left(1 - \frac{\sigma(w,w_i)}{\sigma(w,w^*)}\right)^2 \to 0,\]
 so $D_n \to 0$ in probability and the claim follows from Slutsky's theorem.
\end{proof}

We will use this version of the central limit theorem to establish uniform convergence properties of the random variables $\psi_n(w)$ over compact subsets of $\interior W \bs W_0$.  Our goal is an analogue of \citet[Theorem 5.9]{WatanabeGrey}.  More specifically we will prove the following.

\begin{theorem} \label{psi_convergence_theorem}
 Let $K \sub \interior W \bs W_0$ be a compact subset.  The sequence $\psi_n|_K$ of $C^\omega(K)$-valued random variables converges in distribution in $C(K)$, the space of continuous functions on $K$ with the topology of uniform convergence, to a random variable $\psi|_K$, a centered Gaussian process on $K$ with covariance $\bb E_{x \sim q_{w^*}}[\wt f(w,w^*,x)\,\wt f(v,w^*,x)]$.
\end{theorem}

Our proof strategy---following Watanabe---will combine Prokhorov's theorem, telling us that is we can prove $\psi_n|_K$ is a uniformly tight sequence then it admits a convergent subsequence, with Proposition \ref{psi_CLT} telling us that the sequence $\psi_n(w)$ converges pointwise for $w \in  \interior W \bs W_0$.  We will begin with a sequence of lemmas that will together prove uniform tightness.  Throughout these lemmas we fix a compact subset $K \sub \interior W \bs W_0$.

\begin{lemma} \label{Ls_lemma}
 Let $s>2$ be an even integer.  Suppose that $f(w,w',-)$ is a random variable of class $L^s(q_w)$ for all $w,w' \in W$.  Then $\psi_n(w)$ is also of class $L^s(q_w)$ for all $w \in W$.
\end{lemma}

\begin{proof}
If $n=1$ we have $\psi_1(w)(x) = \frac 1{\sigma(w,w^*)}(\bb E(f(w,w',x)) - f(w,w',x))$.  Since $q_w$ is a probability measure and $f(w,w',-)$ is of class $L^s(q_w)$ so is $\psi_1(w)$.  For $n>1$, $\psi_n(w)$ is a linear combination of $\psi_1(w)$ random variables and therefore is also of class $L^s(p_w)$.
\end{proof}

\begin{lemma} \label{psi_uniform_bound_lemma}
Let $s>2$ be an even integer.  Suppose that $\psi_n(w)$ is a random variable of class $L^s(q_w)$ for all $w \in W$. Then
\[\limsup_{n \to \infty} \|\sup_{w \in K} \psi_n(w)\|_s < \infty.\]
\end{lemma}

\begin{proof}

We will generalize the argument given by in \citet[Theorems 5.7, 5.8]{WatanabeGrey}.  Recall that
\begin{align*}
 \psi_n(w) &= \frac 1{\sqrt{n}} \sum_{i=1}^n \frac{G(w) - f(w,w_i,x_i)}{\sigma(w,w^*)} \\
 &= \frac 1{\sqrt{n}} \sum_{i=1}^n \wt f(w,w_i,x)
\end{align*}
where
\[\wt f(w,w_i,x) = \frac{G(w) - f(w,w_i,x)}{\sigma(w,w^*)}.\]
By Assumption \ref{ls_analytic_assumption} the random variable $\wt f$ is $L^s$-valued real analytic in $(w,w')$, uniformly on $V$: $G(w)$ and $\sigma(w,w^*)$ are analytic in $w$, the latter because $\sigma(w,w^*)^2$ is the integral of the square of an $L^s$-valued real analytic function, and $\sigma(w,w^*) > 0$ on $K$.  Since $w_i \to w^*$ we may assume without loss of generality that all $w_i$ lie inside a polydisk $D \sub W$ of absolute convergence around $w^*$ in the second variable.  Indeed,
\begin{align*}
 \|\sup_{w \in K} \psi_n(w)\|_s &= \left\| \frac 1{\sqrt{n}} \sum_{i=1}^n \sup_{w \in K} \wt f(w,w_i,x)\right\|_s \\
 &\le \left\| \frac 1{\sqrt{n}} \sum_{i=1}^{m-1} \sup_{w \in K} \wt f(w,w_i,x)\right\|_s + \left\| \frac 1{\sqrt{n}} \sum_{i=m}^n \sup_{w \in K} \wt f(w,w_i,x)\right\|_s\\
 \implies \lim_{n \to \infty} \|\sup_{w \in K} \psi_n(w)\|_s &\le \lim_{n \to \infty} \left\| \frac 1{\sqrt{n}} \sum_{i=m}^n \sup_{w \in K} \wt f(w,w_i,x)\right\|_s
\end{align*}
for all $m \ge 1$.

Therefore we may choose a finite set of polydisks covering $K$  with centers $\{z\}$ so that we can expand $\psi_n$ as an absolutely convergent series
\[\psi_n(w) = \frac 1{\sqrt{n}} \sum_{i=1}^n \sum_{\alpha \in \ZZ^d} b_{\alpha}(w_i, x_i)(w-z)^{\alpha}\]
where $b_\alpha(w_i, x)$, the coefficient of $(w-z)^\alpha$ in the expansion of $\wt f(\cdot, w_i, x)$, is of class $L^s(q_{w_i})$ and $\bb E_{x \sim q_{w_i}} b_\alpha(w_i, x) = 0$ for every $\alpha$, because $\bb E_{x \sim q_{w_i}} \wt f(w, w_i, x) = 0$ for all $w$ and the expectation commutes with the absolutely convergent expansion.  Since the $w_i$ lie in the compact subset $D$, $M_\alpha = \sup_i \|b_\alpha(w_i, \cdot)\|_{s, w_i}$ is finite and $\sum_\alpha M_\alpha R^\alpha < \infty$ for $R$ inside the polydisk of convergence.  It is sufficient to establish the bound on $\bb E(\sup_{w \in K}|\psi_n(w)|^s)$ for each of these finitely many polydisks.  On such a polydisk we have
\begin{align*}
 \|\sup_{w \in K}\psi_n(w)\|_s &= \left\|\sup_{w \in K} \sum_{\alpha \in \ZZ^d} (w-z)^{\alpha} \left(\frac 1{\sqrt{n}} \sum_{i=1}^n  b_{\alpha}(w_i, x_i) \right) \right\|_s  \\
 &\le  \left\|\sum_{\alpha \in \ZZ^d} R^\alpha \left(\frac 1{\sqrt{n}} \sum_{i=1}^n  b_{\alpha}(w_i, x_i) \right) \right\|_s  \\
 &\le \sum_{\alpha \in \ZZ^d} R^\alpha \left(\left\|\frac 1{\sqrt{n}} \sum_{i=1}^n  b_{\alpha}(w_i, x_i) \right\|_s \right)  \\
\end{align*}
 where $R$ is an element of $\RR_{>0}^d$ and $R^\alpha = \prod_{j=1}^d R_j^{\alpha_j}$.  We need to verify that this series is uniformly bounded in $n$.

 We will do this by showing
 \[\left\|\frac 1{\sqrt{n}} \sum_{i=1}^n  b_{\alpha}(w_i, x_i) \right\|_s \le sM_\alpha.\]
  Since $s \ge 1$ and the power series expansion for $\psi_n(w)$ is absolutely convergent on the given domain the $\alpha$ sum still converges and we obtain the desired uniform bound.

 Let us consider independent centered random variables $a_i = b_\alpha(w_i, x_i)$ with $\sup_i \|a_i\|_s \le M$ and consider the argument in \citet[Theorem 5.7]{WatanabeGrey}.  Without the assumption that the random variables $X_i$ are identically distributed Watanabe's argument still establishes
 \begin{align*}
  \left\|\frac 1{\sqrt{n}} \sum_{i=1}^n  a_i\right\|_s^s &= n^{-s/2} \bb E\left( \sum_{j\ne 1} \pmat{s \\ j}\left(\sum_{i=1}^{n-1} a_i\right)^{s-j} a_n^{j}\right) \\
  &\le n^{-s/2}\sum_{j\ne 1} \pmat{s \\ j} \left\|\sum_{i=1}^{n-1} a_i \right\|_s^{s-j}M^j \\
  &= \frac 1{\sqrt{n}} \left(\left\|\sum_{i=1}^{n-1} a_i \right\|_s + M\right)^s - s \left\|\sum_{i=1}^{n-1} a_i \right\|_s^{s-1} M.
 \end{align*}
  The argument in \emph{loc. cit.} then establishes that
 \begin{align*}
  y_{n+1} &\le F_n(y_n) \\
  \text{where } y_n &= \bb E\left(\left(\frac 1{M\sqrt n} \sum_{i=1}^n a_i \right)^s\right) \\
  \text{and } F_n(y) &= \frac y{1+s/2n}\left(1 + \frac{s(s-1)}{2ny^{2/s}}\left(1+\frac 1{\sqrt n y^{1/s}}\right)^{s-2}\right),
 \end{align*}
 a monotonic function of one variable.  Finally Watanabe checks that $F_n((s-1)^s) \le (s-1)^s$.  The base case is $y_1 = (\|a_1\|_s/M)^s \le 1 < (s-1)^s$, so by induction we will have that $y_n \le (s-1)^s$ for all $n$, and therefore that
 \[\left\|\frac 1{\sqrt{n}} \sum_{i=1}^n  b_{\alpha}(w_i, x_i) \right\|_s \le (s-1) M_\alpha \le sM_\alpha.\]
\end{proof}

\begin{lemma}\label{lemma:rohan}
 For every $\eps > 0$ there exists a compact subset $C_\eps \sub C(K)$ for which
 \[\bb P(\psi_n|_K \in C_\eps) > 1-\eps.\]
\end{lemma}

\begin{proof}
We will prove this in two steps.  First choose an even integer $s>2$ as Lemma \ref{psi_uniform_bound_lemma} and use the lemma to find $C > 0$ so that
\[\bb E(\sup_{w \in K} |\psi_n(w)|^s) < C\]
for all $n$.  Therefore Markov's lemma implies that, for all $M > 0$,
\begin{align*}
 \bb P(\sup_{w \in K} |\psi_n(w)| \ge M) &< \frac{C}{M^s} \\
 \implies \bb P(\sup_{w \in K} |\psi_n(w)| < M &\ge 1-\frac{C}{M^s}
\end{align*}
We therefore define $C'_\eps$ to be the set
\[C'_\eps = \{\psi \in C(K) \colon \sup_{w \in K} |\psi(w)| \le \left(\frac{C}{\eps}\right)^{1/s}\},\]
a closed ball in the $L^\infty$ topology on $C(K)$.  Now, these closed balls are not compact, so we will need to find a compact subspace $C_\eps \sub C'_{\eps/2}$ of measure $>1-\eps$.  Since $C'_{\eps/2}$ is already uniformly bounded, by the Arzel\`a--Ascoli theorem it suffices to find an equicontinuous subspace of measure $>1-\eps$.

Choose a Riemannian metric on $W$ and let
\[D_{k,\delta} = \{\psi \in C(K) \colon \sup_{\|w - w'\| < \delta} |\psi(w) - \psi(w')| \le \frac 1 k\},\]
a closed set, so that the intersection below is closed as well as bounded and equicontinuous.
We will find a sequence of values $\delta_k$ so that
\[\bb P(\psi_n \in D_{k,\delta_k}) \ge 1- 2^{-k-1}\eps,\]
so that $C_\eps = C'_{\eps/2} \cap \bigcap_{k=1}^\infty D_{k,\delta_k}$ is a compact subset of measure greater than $1-\eps$.

So, to conclude, we check
\begin{align*}
 \bb P(\sup_{\|w - w'\| < \delta} |\psi_n(w) - \psi_n(w')| \ge \frac 1 k) &= P(\sup_{\|w - w'\| < \delta} |\psi_n(w) - \psi_n(w')|^s \ge \frac 1 {k^s}) \\
 &\le k^s \bb E(\sup_{\|w - w'\| < \delta} |\psi_n(w) - \psi_n(w')|^s) \\
 &\le k^s \delta^s \bb E(\sup_{w \in K} \|\nabla \psi_n(w)\|^s)
\end{align*}
by the mean value theorem (applied within convex coordinate charts covering $K$, after shrinking $\delta$).  Since $\psi_n(w) \in L^s$ so is $\nabla \psi_n$, so the expectation value on the last line above is finite, and therefore
\[\bb P(\sup_{\|w - w'\| < \delta} |\psi_n(w) - \psi_n(w')| \ge \frac 1 k) \to 0\]
as $\delta \to 0$, which implies the desired claim.
\end{proof}

\begin{remark}
 The argument given for this claim in \citet[Example 5.3]{WatanabeGrey} requires an additional step since closed balls in the $L^\infty$ topology on infinite spaces are not compact.  This point and the correction used above were communicated to us by Rohan Hitchcock.
\end{remark}

\begin{corollary} \label{uniformly_tight_corollary}
The sequence $\psi_n|_K$ of $C^\omega(K)$-valued random variables is uniformly tight as a sequence of $C(K)$-valued random variables.
\end{corollary}

We may now prove the main result of this section: \cref{psi_convergence_theorem}.

\begin{proof}
Applying Prokhorov's theorem, Corollary \ref{uniformly_tight_corollary} tells us that $\psi_n(w)$ has a convergent subsequence.  Now Proposition \ref{psi_CLT} says that $\psi_n(w)$ converges in distribution pointwise for $w \in K$.  Likewise for its associated finite-dimensional distributions: for finitely many $w_1, \ldots, w_k \in K$ and $c \in \RR^k$ the linear combination $\sum_j c_j \psi_n(w_j) = n^{-1/2} \sum_i \sum_j c_j \wt f(w_j, w_i, x_i)$ is again a normalized sum of independent centered terms, so the Lyapounov computation of Proposition \ref{psi_CLT} applies to it and gives convergence in distribution to $N(0, c^{\mathsf T} \Sigma c)$ with $\Sigma_{jl} = \bb E_{x \sim q_{w^*}}[\wt f(w_j,w^*,x)\,\wt f(w_l,w^*,x)]$.  By the Cram\'er--Wold device, a sequence of random vectors converges in distribution as soon as every fixed linear combination of its coordinates converges, so $(\psi_n(w_1), \ldots, \psi_n(w_k))$ converges in distribution to the finite-dimensional distribution of the centered Gaussian process on $K$ with covariance $\bb E_{x \sim q_{w^*}}[\wt f(w,w^*,x)\,\wt f(v,w^*,x)]$.  The finite-dimensional distributions determine the law of a $C^\omega(K)$-valued random variable, so every convergent subsequence has the same limit, and therefore $\psi_n(w)$ converges in $C(K)$ to a random variable $\psi|_K$.
\end{proof}

\subsubsection{Resolution of Singularities}
Our next goal is to establish part (1) of Theorem \ref{main_theorem}.  That is, we will give a local (in parameter space) description of the importance weighted estimator $G_n$ and the empirical process $\psi_n$ in a standard form that is comparatively easy to manipulate.  We will do this by fixing a \emph{log resolution} of the zero set $W_0 \sub W$ of the loss function.  We recall the fundamental construction as originally developed by \citet{Hironaka}.

\begin{definition} \label{log_resolution_def}
 A local \emph{log resolution} of $G$ at a point $w^* \in W_0 \sub W$ is a proper real analytic map $\varpi \colon \mc M \to W$ of real analytic manifolds with corners such that
 \begin{enumerate}
  \item $\varpi$ restricts to an isomorphism on $\varpi^{-1}(\interior W \bs W_0)$.
  \item We can choose a chart $U$ for $\mc M$ centered around any point in $\varpi^{-1}(w^*)$ with coordinate $u$ in which
  \[ \det \mr{Jac}(\varpi)(u) = b(u) u_1^{h_1} \cdots u_d^{h_d}\]
  for an invertible analytic function $b$ and
  \[G(\varpi(u)) = \begin{cases}
                    \pm u_1^{2k_1} \cdots u_d^{2k_d} &\text{ if } w^* \in \interior W \\
                    \pm u_1^{\ell} u_2^{2k_d} \cdots u_d^{2k_d} &\text{ if } w^* \in \dd W
                   \end{cases}\]
  where in the latter case $U$ is isomorphic to a half-space in $\RR^d$ and $u_1$ is a normal coordinate to the boundary.  The constants $h_i$, $k_i$ and $\ell$ are non-negative integers.  In the boundary case with $\ell$ odd we reparametrize by $u_1 = v_1^2$ on $v_1 \ge 0$, which makes the loss exponent $2\ell$ and the Jacobian exponent $2h_1 + 1$ (up to the unit $2$) and leaves $(h_1 + 1)/\ell$, hence the learning coefficient, unchanged.
 \end{enumerate}
\end{definition}

Choose a point $w^* \in W_0$, a local log resolution $\varpi$ and a chart $U$ around this minimum.  Let $U_0 = \varpi^{-1}(W_0) \cap U$ and let $K$ be a compact subset of $\interior U \bs U_0$.  In this chart our importance weighted loss may be usefully expressed in a standard way---in normal crossings form---parallelling \citet[Main Theorem 1]{WatanabeGrey}.  We will begin by representing the random variable $f$ from Definition \ref{importance_sampling_def} in normal crossings form.

\begin{prop} \label{importance_scaled_samples_form_prop}
 Write $g(x)q_{\varpi(u)}(x) - h(x) = \sum_\alpha c_\alpha(x)\, u^\alpha$ for the Taylor expansion in $u$ around $0$, and suppose that for some polydisk $D_\rho = \{|u_i| < \rho_i\}$ around $0$ the function $\sum_\alpha |c_\alpha(x)|\, \rho^\alpha$ is integrable over $X$.  There exists real analytic functions $a, b$ on $X \times U \times W$ so that we may locally expand
 \[f(\varpi(u), w', x) = a(u, w', x)u_1^{\ell/2} u_2^{k_2} \cdots u_d^{k_d} +b(u, w', x),\]
 so that for all $w' \in W$
 \begin{align*}
 \int_X a(u, w', x)q_{w'}(x) \d x &= u_1^{\ell/2} u_2^{k_2} \cdots u_d^{k_d} \\
 \int_X b(u, w', x)q_{w'}(x) \d x &= 0,
 \end{align*}
 and where all Taylor coefficients of $b$ of order $\ge (\ell/2, k_2, \ldots, k_d)$ vanish identically.
\end{prop}

\begin{proof}
We obtain $a,b$ by forming the Taylor expansion in $u$ of the analytic function $f(\varpi(u), w',x)$ and splitting the resulting sum into the sum of those Taylor terms of order at least \[u_1^{\lceil\ell/2\rceil}u_2^{k_2} \cdots u_d^{k_d},\] and the sum of those Taylor terms of smaller degree in at least one variable.  Having done so, by Lemma \ref{importance_weighting_mean_lemma} and the definition of the log resolution we have
\begin{align*}
 \int_X (a(u, w', x)u_1^{\ell/2} u_2^{k_2} \cdots u_d^{k_d} + b(u, w', x))q_{w'}(x) \d x &= \bb E(f(u, w', x)) \\
 &= G(\varpi(u)) \\
 &= \pm u_1^{\ell} \cdots u_d^{2k_d}.
\end{align*}
We obtain the desired expressions by comparing the orders of Taylor coefficients on the two sides and replacing $a$ by $-a$ if $G(\varpi(u)) = -u_1^{\ell} \cdots u_d^{2k_d}$.
\end{proof}

\begin{theorem} \label{general_form_importance_weight_theorem}
 In the local chart $U$ we may express the importance weighted loss as
 \[G_n(\varpi(u)) = u^{2k} + \frac 1{\sqrt n}\left(u^k \xi_n(u) + \omega_n(u)\right)\]
 where $u^{2k} = u_1^{\ell} u_2^{2k_2} \cdots u_d^{2k_d}$ and $u^k = u_1^{\ell/2} u_2^{k_2} \cdots u_d^{k_d}$.
 The random variables
 \[\xi_n(u) + \omega_n(u)u^{-k} = \frac 1{\sqrt n} \left(\sum_{i=1}^n a(u, w_i, x_i) - \bb E(a(u, w_i, x_i))\right) + \frac 1{\sqrt n} \left(\sum_{i=1}^n b(u, w_i, x_i)\right)u^{-k}\]
 converge to a Gaussian in distribution uniformly in $K$.
\end{theorem}

\begin{proof}
By Theorem \ref{psi_convergence_theorem} and Proposition \ref{psi_CLT} we know that $\sqrt n (G(w) - G_n(w)) \to \mc N(0,\sigma(w,w^*))$ in distribution uniformly over compact subspaces of $\interior W \bs W_0$.  In particular for the resolution $\varpi$ we have
\begin{align*}
 \sqrt n(G(\varpi(u)) - G_n(\varpi(u)))) &\to \mc N(0,\sigma(w,w^*)^2) \\
 \implies \frac 1{\sqrt n} \sum_{i=1}^n \left( u^{2k} - a(u, w_i, x_i)u^k - b(u, w_i, x_i) \right) &\to \mc N(0,\sigma(\varpi(u),w^*)^2) \\
  \Leftrightarrow -\xi_n(u)u^k - \omega_n(u) &\to \mc N(0,\sigma(\varpi(u),w^*)^2) \\
  \Leftrightarrow \xi_n(u) + \omega_n(u)u^{-k} &\to \mc N(0,\sigma(\varpi(u),w^*)^2u^{-2k})
\end{align*}
in distribution uniformly in $K$.
\end{proof}

In order to proceed from here we will establish situations where the summands $\omega_n$ converge vanish.

\subsubsection{Free Energy Asymptotics} \label{free_energy_section}
Let us now return to Theorem \ref{main_theorem}, in this section we will address part (2), concerning the asymptotic behavior of the generalized posterior distribution.  In order to establish clear asymptotics we will need to guarantee that the lower order terms in the importance weighted loss estimators---the terms $\omega_n$ from Theorem \ref{general_form_importance_weight_theorem}---do not contribute to the posterior.  This is not likely to hold in general, we will establish it under the condition that Assumption \ref{key_assumption} holds, i.e., that the integrand $g(x)q_w(x) - h(x)$ is almost everywhere non-negative for $w$ near $W_0$, so that in particular if $w^* \in W$ minimizes the loss then the set
\[\{x \colon g(x)q_{w^*}(x) - h(x) \ne 0\} \sub X\]
 has Lebesgue measure zero (Lemma \ref{g_positive_given_assumption}); for local minima that are not global we assume the following condition instead.

\begin{definition} \label{rfv_def}
Let $w_0 \in W$ be a local minimum of $G$, with value $G_0$ on an open neighborhood $U$.  We say that $w_0$ satisfies the \emph{relatively finite variance} condition if there is a constant $C$ such that
\[\bb E_{x \sim q_{w'}}\left[\left(f(w,w',x) - f(w_0,w',x)\right)^2\right] \le C\,(G(w) - G_0)\]
for all $w \in U$ and all $w'$ in the compact set $V \sub \interior W \bs W_0$ containing the behavior sequence $(w_i)$.  Note that $f(w,w',x) - f(w_0,w',x) = (q_w(x) - q_{w_0}(x))g(x)/q_{w'}(x)$ has expected value $G(w) - G_0$ under $q_{w'}$, so the condition compares the second moment of the per-sample contribution to $G_n(w) - G_n(w_0)$ with its mean; it is the condition of \citet{WatanabeGreen} for the centered estimator.
\end{definition}

 Our initial aim is to show the following.
 \begin{prop} \label{omega_zero_prop}
 If $q_w$ is a probability distribution for all $w$ and Assumption \ref{key_assumption} holds then $\omega_n = 0$ for all $n$.
\end{prop}

\begin{remark} \label{rfv_remark_original}
 Proposition \ref{omega_zero_prop} is very close to the condition of \emph{relatively finite variance} that appears in \citet{WatanabeGreen}.  In the general case where $q_w$ is merely a signed measure (which includes the density estimation problem originally considered by Watanabe) the relatively finite variance condition is what is needed to establish $\omega_n = 0$.  This is, however, implied by Assumption \ref{key_assumption} in the positive case.  Indeed, recall that the function $f$ has \emph{relatively finite variance} if
 \[\sup_{w \in \interior W \bs W_0} \frac {\mr{Var}(f(w,w',x))}{\bb E(f(w,w',x))} < \infty \]
 for all $w' \in \interior W \bs W_0$.  If we work in a local chart $U$ in a resolution around a minimum $w^*$ of the loss then $\bb E(f(w,w',x)) = u^{2k}$, and we will see in the proof of Proposition \ref{omega_zero_prop} that under the assumption $f(w,w',x)$ also has Taylor expansion with leading term of order $2k$, so the ratio of the variance and expectation also converges to zero along as $u \to 0$.

 We will assume below that $q_w$ is a probability distribution for all $w$; in particular that it is non-negative.  If $q_w = -\log p_w$ for a model $p_w$, as in the distribution learning example \ref{supervised_example}, this is not generally true.  In this case Proposition \ref{omega_zero_prop} still holds in the realizable case by \citet[Theorem 6.1]{WatanabeGrey} or under a more general relative finite variance condition as in \citet[Theorem 8]{WatanabeGreen}.
\end{remark}

\begin{lemma} \label{g_positive_given_assumption}
 Suppose Assumption \ref{key_assumption} holds and let $w^* \in W$ minimize the loss.  Then the set $\{x \colon g(x)q_{w^*}(x) - h(x) \ne 0\}$ has Lebesgue measure zero.  Consequently $f(w^*, w', x) = 0$ for almost every $x$ and every $w' \in \interior W$, and $G_n(w^*) = 0$ almost surely for every $n$ and every behavior sequence.
\end{lemma}

\begin{proof}
The integrand of $G(w^*) = 0$ is almost everywhere non-negative by Assumption \ref{key_assumption}, since $w^* \in N$, and has zero integral, hence vanishes almost everywhere.  Dividing by $q_{w'}(x) > 0$ gives the second claim, and the third follows since each summand $f(w^*, w_i, x_i)$ of $G_n(w^*)$ vanishes almost surely.
\end{proof}

We now prove Proposition \ref{omega_zero_prop}.

\begin{proof}
 Let $w^* \in W$ be a point where $G(w^*) = 0$, and let $X_{w^*} = \{x \colon g(x)q_{w^*}(x) - h(x) = 0\} \sub X$.  By Lemma \ref{g_positive_given_assumption} we know that $X \bs X_{w^*}$ is a set of Lebesgue measure zero.  We choose a log resolution $\varpi$ and a chart $U$ with coordinate $u$ for which $\varpi(0) = w^*$.  Then
 \begin{align*}
  u^{2k} &= G(\varpi(u)) \\
  &= \int_X g(x)q_{\varpi(u)}(x) - h(x) \d x \\
  &= \int_{X_{w^*}} g(x)q_{\varpi(u)}(x) - h(x) \d x \\
  &= \int_{X_{w^*}} g(x) \left(\sum_\alpha b_\alpha(x)u^\alpha \right) - h(x) \d x \\
  &= \left( \int_{X_{w^*}} g(x)b_0(x) - h(x) \d x\right) +\sum_{\alpha>0} \left(\int_{X_{w^*}} g(x)b_\alpha(x)\right) u^\alpha \\
  &= \sum_{\alpha\ge 0} \int_{X_{w^*}} c_\alpha(x) u^\alpha \d x,
 \end{align*}
 where $c_0 = g b_0 - h$ and $c_\alpha = g b_\alpha$ for $\alpha > 0$.  We used absolute convergence of the Taylor series to exchange the sum and the integral on a polydisk $D_\rho = \{|u_i| < \rho_i\}$ around $0$ on which $\int_X \sum_\alpha |c_\alpha(x)|\, \rho^\alpha \d x < \infty$.  
 
 Write $g(x)q_{\varpi(u)}(x) - h(x)= \phi_x(u)$ for brevity.  By Assumption \ref{key_assumption}, shrinking the chart as needed, $\phi_x(u) \ge 0$ for almost every $x$ and all $u$ in the chart, and $\phi_x(u) = 0$ for almost every $x$ whenever some coordinate $u_j$ with $k_j > 0$ vanishes, since then $\varpi(u) \in W_0$ (Lemma \ref{g_positive_given_assumption}).  Fix such a $j$ and fix the other coordinates $u_{-j}$ at generic non-zero values in the chart.  The function 
 \[s \mapsto \phi_x(u_{-j}, s)\] 
 is non-negative, real analytic on an interval around $0$ and vanishes at $s = 0$.  Therefore wherever it is not identically zero its leading Taylor coefficient in $s$ is positive.  Its $s^\beta$ coefficient integrates over $x$ to the $s^\beta$ coefficient in $u_{-j}^{2k_{-j}} s^{2k_j}$, which is zero for $\beta < 2k_j$.  At the least order $\beta_0$ at which a set of $x$ of positive measure has a non-zero coefficient the integrand is non-negative and positive on that set, which is a contradiction unless $\beta_0 \ge 2k_j$.  Hence for almost every $x$ every coefficient $c_\alpha(x)$ with $\alpha_j < 2k_j$ vanishes for these $u_{-j}$, and by analyticity in $u_{-j}$ for all $u_{-j}$.  Repeating this argument for all $j$ shows that $c_\alpha(x) = 0$ unless $\alpha \ge 2k$ componentwise.
 
 So, we conclude
 \begin{align*}
  f(\varpi(u),w',x)  &= \frac{q_{\varpi(u)}(x)}{q_{w'}(x)}g(x) -  \frac 1{q_{w'}(x)}h(x) \\
  &= \frac 1{q_{w'}(x)} \left(g(x)q_{w^*}(x) - h(x) + \sum_{\alpha \ge 2k} c_{\alpha}(x)u^{\alpha}\right) \\
  &= f(w^*,w',x) + \frac 1{q_{w'}(x)} \left(\sum_{\alpha \ge 2k} c_{\alpha}(x)u^{\alpha}\right).
 \end{align*}
 Finally, Lemma \ref{g_positive_given_assumption} implies that $f(w^*,w',x)=0$, and therefore the Taylor series of $f(\varpi(u),w',x)$ in $u$ is concentrated in degrees at least $2k$, and therefore in particular $\omega_n = 0$ for all $n$.
\end{proof}

\begin{lemma} \label{rfv_from_positivity}
Suppose Assumption \ref{key_assumption} holds, $g$ is bounded, and the density ratios are bounded near the optimal set: $q_w(x) \le M q_{w'}(x)$ for $w \in N$, $w' \in V$ and almost every $x$.  Then for $w \in N$ and $w' \in V$, $0 \le f(w,w',x) \le M \|g\|_\infty$ almost everywhere and
\[\bb E_{x \sim q_{w'}}\left[f(w,w',x)^2\right] \le M\|g\|_\infty\, G(w).\]
In particular every optimal parameter satisfies Definition \ref{rfv_def}, with $f(w^*,w',x) = 0$ by Lemma \ref{g_positive_given_assumption}.
\end{lemma}
\begin{proof}
Since $h \ge 0$, $0 \le f = (gq_w - h)/q_{w'} \le g q_w/q_{w'} \le \|g\|_\infty M$ almost everywhere, and $\bb E[f^2] \le \sup f \cdot \bb E[f] = M\|g\|_\infty G(w)$ by Lemma \ref{importance_weighting_mean_lemma}.
\end{proof}

\begin{corollary} \label{full_chart_convergence}
Under the hypotheses of Proposition \ref{omega_zero_prop}, the convergence in distribution of Theorem \ref{general_form_importance_weight_theorem} holds uniformly on compact subsets of the chart $U$ that meet $U_0$.
\end{corollary}
\begin{proof}
By Proposition \ref{omega_zero_prop}, $f(\varpi(u),w',x) = u^{2k} c(x,u)/q_{w'}(x)$ with $c$ analytic on the chart and $\int_X c(x,u) \d x = 1$, so $a(u,w',x) = u^k c(x,u)/q_{w'}(x)$ is analytic in $(u,w')$ on the whole chart with Taylor coefficients in $L^s(q_{w'})$, and $\xi_n(u) = \frac 1{\sqrt n}\sum_i (a(u,w_i,x_i) - u^k)$ is a normalized sum of independent centered summands.  The proofs of Lemmas \ref{Ls_lemma} and \ref{psi_uniform_bound_lemma}, of the tightness lemma following them, of Corollary \ref{uniformly_tight_corollary} and of Theorem \ref{psi_convergence_theorem} apply verbatim with $\xi_n$ in place of $\psi_n$.
\end{proof}

\begin{prop} \label{rfv_standard_form}
Let $w_0$ be a local minimum of $G$ satisfying Definition \ref{rfv_def}, and let $\varpi$ be a log resolution of $G - G(w_0)$ at $w_0$ with a chart $U$ in which $G(\varpi(u)) - G(w_0) = u^{2k}$.  Then there is a function $\wt a(u,w',x)$, analytic in $u$ with values in $L^2(q_{w'})$, such that
\[f(\varpi(u),w',x) - f(w_0,w',x) = u^k\, \wt a(u,w',x), \qquad \int_X \wt a(u,w',x)\, q_{w'}(x) \d x = u^k,\]
and hence $G_n(\varpi(u)) - G_n(w_0) = u^{2k} + \frac 1{\sqrt n} u^k \wt\xi_n(u)$ with $\wt \xi_n(u) = \frac 1{\sqrt n}\sum_i (\wt a(u,w_i,x_i) - u^k)$.  If $\wt a$ is moreover analytic in $(u,w')$ with $L^s$-valued coefficients then $\wt \xi_n$ converges in distribution on the chart as in Corollary \ref{full_chart_convergence}.
\end{prop}
\begin{proof}
Let $F(u) = f(\varpi(u),w',\cdot) - f(w_0,w',\cdot) \in L^2(q_{w'})$, an analytic function of $u$ with $\|F(u)\|^2 \le C(G(\varpi(u)) - G(w_0)) = C u^{2k}$ by Definition \ref{rfv_def}.  An $L^2$-valued analytic function bounded by a constant times $|u^k|$ is divisible by $u^k$: fixing all coordinates but $u_j$ at generic values, the one-variable function is bounded by a constant times $|u_j|^{k_j}$, so its Taylor coefficients of order below $k_j$ vanish, and by analyticity this holds for all values of the other coordinates; repeating over $j$ gives $F(u) = u^k \wt a(u)$ with $\wt a$ analytic. The identity $\int \wt a\, q_{w'} = u^k$ follows from $\int F\, q_{w'} = u^{2k}$ where $u^k \ne 0$ and by continuity elsewhere, and the expression for $G_n(\varpi(u)) - G_n(w_0)$ is then immediate.
\end{proof}

We may now study the asymptotic behavior of the generalized posterior distribution in terms of the geometry of the loss function $G$. We recall from Definition \ref{posterior_def} the generalized tempered posterior distribution $\Omega$, the evidence $Z_{n,\beta}(U)$ and free energy $F_{n,\beta}(U)$.

\begin{theorem} \label{evidence_asymptotic_theorem}
Suppose that $q_w$ is a family of probability distributions and Assumption \ref{key_assumption} holds.  Let $w_0$ be a local minimum of $G$ satisfying Definition \ref{rfv_def} (every optimal parameter does, by Lemma \ref{rfv_from_positivity}), let $U \sub W$ be an open neighborhood of $w_0$ small enough that $G \ge G(w_0)$ on $U$ and covered by finitely many charts of a local log resolution of $G - G(w_0)$ (Definition \ref{log_resolution_def}), and let $(\lambda, m) = (\mr{rlct}_U(G - G(w_0)), \mr{rlcm}_U(G - G(w_0)))$ denote the real log canonical threshold and multiplicity of the loss function $G - G(w_0)$ on the subset $U$.  Then
\begin{align*}
 Z_{n,\beta}(U) &= e^{-n\beta G_n(w_0)}\left(C(\beta, \varphi|_U) n^{-\lambda}(\log n)^{m-1} + o_{\mr P}(n^{-\lambda}(\log n)^{m-1})\right)
\end{align*}
where $C(\beta, \varphi|_U)$ is constant in $n$ when $w_0$ is an optimal parameter, and in general a positive random variable that does not depend on $n$.
\end{theorem}

\begin{proof}
Given Proposition \ref{omega_zero_prop}, Theorem \ref{general_form_importance_weight_theorem} and either Corollary \ref{full_chart_convergence} or, for a local minimum that is not global, Proposition \ref{rfv_standard_form} this follows from Watanabe's argument in \citet[Theorem 6.7]{WatanabeGrey}, the function applied to $G_n - G_n(w_0)$ on Watanabe's partition of a neighborhood of $w_0$ into closed cubes in the charts \citep[Theorem 6.5]{WatanabeGrey}.  Off the chart images $G - G(w_0) \ge c > 0$ on a compact set, on which $G_n - G_n(w_0) \ge c/2$ with probability tending to one by Lemma \ref{psi_uniform_bound_lemma} and Chebyshev's inequality at $w_0$.  Thus that region contributes a term of order $O_{\mr P}(e^{-n\beta c/2})$ and is negligible.  The factor $e^{-n\beta G_n(w_0)}$ is exact, and equals one at an optimal parameter by Lemma \ref{g_positive_given_assumption}.
\end{proof}

\begin{remark}
 One can give an explicit integral formula for $C(\beta, \varphi)$ as a sum over local charts in a log resolution, see \citet[Section 6.3]{WatanabeGreen} or \citet[Section 6.2]{WatanabeGrey} for such an expression.
\end{remark}

\begin{corollary} \label{free_energy_corollary}
 We have convergence in distribution
 \[F_{n,\beta}(U) - n\beta G_n(w_0) - \lambda \log n + (m-1) \log \log n \to F_{\beta}(U)\]
 for a random variable $F_{\beta}(U)$ depending on $\beta$ and the prior $\varphi$.
\end{corollary}

\subsubsection{Expected Total Loss and the Widely Applicable Bayesian Information Criterion} \label{WBIC_section}
We will now conclude our technical arguments by establishing part (3) of Theorem \ref{main_theorem}, the analogue of the \emph{widely applicable Bayesian information criterion} \citet{WatanabeWBIC}, explaining how the generalized posterior selects between different critical regions of $W_0$ as the number $n$ of datapoints increases.

We will establish the following.

\begin{theorem} \label{WBIC_theorem}
Suppose that $q_w$ is a family of probability distributions and Assumption \ref{key_assumption} holds. Let $\beta \colon \bb N \to \RR_{>0}$ be a positive function on the natural numbers such that $\beta(n)$ converges to zero as $n \to \infty$ and $n\beta(n) \to \infty$.  Let $w^*$ be a local minimum of the loss satisfying Definition \ref{rfv_def} and let $U$ be a sufficiently small open neighborhood of $w^*$ in $W$, covered by finitely many charts of a local log resolution of $G - G(w^*)$.  Then
\[\frac {n\beta}{Z_{n,\beta}(U)}\int_U (G_n(w) - G_n(w^*)) \exp\left(-n \beta G_n(w)\right) \varphi(w) \d w \to \lambda\]
in probability as $n \to \infty$, where $\lambda = \mr{rlct}_U(G - G(w^*))$ is the real log canonical threshold of the loss function.  If $w^*$ is an optimal parameter the conclusion holds for any $\beta(n)$ converging to a finite limit with $n\beta(n) \to \infty$.
\end{theorem}

\begin{proof}
We will follow the argument in \citet[Theorem 4]{WatanabeWBIC} closely.  It suffices to work locally in $W$; choose a local log resolution $\varpi \colon \mc M \to W$ and a local chart $V$ with local coordinate $u$ centered around a preimage of $w^*$.  Choose the chart so that $G_0 = G(\varpi(0))$ is the minimum value of $G \circ \varpi$ on $V$.  Let $\wt f(w,w',x)= f(w,w',x) - f(w^*,w',x)$ and $\wt G_n = G_n - G_n(w^*)$.  Then
\begin{align*}
 &\frac {n\beta}{Z_{n,\beta}(U)}\int_{\varpi(V)} (G_n(w) - G_n(w^*)) \exp\left(-n \beta G_n(w)\right) \varphi(w) \d w \\
 &\quad = \frac {n\beta}{\wt Z_{n,\beta}(U)}\int_{\varpi(V)} \wt G_n(w) \exp\left(-n \beta \wt G_n(w)\right) \varphi(w) \d w
\end{align*}
where $\wt Z_{n,\beta}$ denotes the evidence with respect to the normalized posterior \[\exp\left(-n \beta \wt G_n(w)\right) \varphi(w) \d w.\]

We will study the numerator and denominator of our WBIC expression separately.  Since we have established Proposition \ref{omega_zero_prop} (at an optimal parameter) or Proposition \ref{rfv_standard_form} (at a local minimum satisfying Definition \ref{rfv_def}), together with Corollary \ref{full_chart_convergence}, Watanabe's calculations of these two expressions continue to apply in our setting (the uniform convergence of $G_n$ to $G$ on the compact complement of the charts follows from Lemma \ref{psi_uniform_bound_lemma}, and the chart expansions have $L^s$ coefficients by Assumption \ref{ls_analytic_assumption}); that is, by the argument in the proof of \citet[Theorem 4]{WatanabeWBIC} we have asymptotic behavior
\begin{align*}
 &\int_{\varpi(V)} \wt G_n(w) \exp\left(-n \beta \wt G_n(w)\right) \varphi(w) \d w \\
 &\qquad = C(\varphi)\log(n\beta)^{m-1}(n\beta)^{-\lambda-1} \Gamma(\lambda+1) + o_{\mr P}(\log(n\beta)^{m-1}(n\beta)^{-\lambda-1}) \\
 \text{and } &\wt Z_{n,\beta} = \int_{\varpi(V)} \exp\left(-n \beta \wt G_n(w)\right) \varphi(w) \d w \\
 &\qquad = C(\varphi) \log(n\beta)^{m-1}(n\beta)^{-\lambda} \Gamma(\lambda) + o_{\mr P}(\log(n\beta)^{m-1}(n\beta)^{-\lambda})
\end{align*}
as $n \to \infty$.  So
\begin{align*}\frac {n\beta}{\wt Z_{n,\beta}(U)}\int_{\varpi(V)} \wt G_n(w) \exp\left(-n \beta \wt G_n(w)\right) \varphi(w) \d w &\to n\beta \frac{C(\varphi)\log(n\beta)^{m-1}(n\beta)^{-\lambda-1} \Gamma(\lambda+1)}{C(\varphi) \log(n\beta)^{m-1}(n\beta)^{-\lambda} \Gamma(\lambda)} \\ &= \lambda\end{align*}
as $n \to \infty$ as required.
\end{proof}

\begin{remark}\label{remark:fef}
\citet{WatanabeWBIC} uses the specific form $\beta = \beta_0/\log(n)$ for comparison with the \emph{free energy} of the Bayesian posterior.  That is for $\mc F_{n,\beta_0}(U)$ as defined in Corollary \ref{free_energy_corollary} we have
 \[\mc F_{n,\beta_0}(U) - \bb E_n^\beta(nG_n(w)|_U) = o_P(\log n) \]
 by Theorem \ref{WBIC_theorem}.  The term \emph{widely applicable Bayesian information criterion} (WBIC) is used for this expression since it generalizes the classical Bayesian information criterion \citep{SchwarzBIC}, which this result recovers in the regular case.
\end{remark}

\section{Theory of LLC Estimation in Reinforcement Learning} \label{appendix:llc_estimation_appendix}

In this appendix, we offer a self-contained introduction to LLC estimation, following \citet{LLC}, but adapted to the RL context.

\begin{definition}
 Let $w^* \in W$ be a local minimum of the loss function $G$.  The \emph{local learning coefficient} of the optimization problem associated to $G$ is the real log canonical threshold of $G - G(w^*)$ at $w^*$.  We will denote the local learning coefficient as $\lambda(w^*)$.  Often we will just write $\lambda$, keeping the choice of minimum implicit.
\end{definition}

\begin{remark}
 If $U \sub W$ is a neighborhood of $w^*$ on which $G \ge G(w^*)$ we will sometimes write
 \[\lambda(U) = \mr{rlct}_U(G - G(w^*)) = \inf_{w \in U} \mr{rlct}_w(G - G(w^*)).\]
 If we choose a metric on $W$ then we can view the LLC as arising by $\lambda(w^*) = \lim_{\eps \to 0} \lambda(B_\eps(w^*))$.
\end{remark}

As we have seen from Theorem \ref{evidence_asymptotic_theorem} the local learning coefficient controls the local behaviour of the generalized Bayesian posterior distribution near the point $w^*$ in the large $n$ limit.  In general when $\dim W$ is very large we cannot expect to compute $\lambda$ exactly, so in this section we will discuss statistical estimators that we will use to substitute for it.  Our estimators are based on the local learning coefficient estimators developed in \citet{LLC}.

\subsection{Asymptotically Unbiased Estimators}
We will begin with those estimators for which we can establish good theoretical properties rigorously, and generalize from there to estimators over which we have less theoretical control but which are more practical to estimate in realistic examples.  To start with we can use the WBIC---Theorem \ref{WBIC_theorem}---to define asymptotically unbiased estimators for the local learning coefficient.

Given a sequence of samples $x_i \sim q_{w_i}$ as in Definition \ref{importance_sampling_def}, choose $\beta > 0$ and let $\Omega_{n,\beta}$ be the normalized posterior distribution from Definition \ref{posterior_def}.  Choose an open neighborhood $U$ of $w^*$ and a natural number $T$ and let $y_1, \ldots, y_T$ be independent random variables from the restriction of the generalized tempered posterior to $U$: $y_j \sim \Omega_{n,\beta}|_U$.

\begin{definition}
Fix an analytic prior distribution $\varphi$ on $W$.  Define a random variable, the \emph{WBIC LLC estimator}, by
\[\widehat \lambda_{\mr{WBIC},\varphi}(U) = \frac {n\beta}{T} \sum_{j=1}^T (G(y_j) - G(w^*)).\]
\end{definition}

\begin{prop} \label{WBIC_estimator_convergence}
 The WBIC LLC estimator $\widehat \lambda_{\mr{WBIC},\varphi}(U)$ converges in probability to $\lambda(U)$, for $w^*$ and $\beta(n)$ as in Theorem \ref{WBIC_theorem}:
 \[\lim_{n \to \infty} \lim_{T \to \infty}  \widehat \lambda_{\mr{WBIC}}(U)\to_P \lambda(U)\]
 for all choices $\varphi$ of positive analytic prior.
\end{prop}

\begin{proof}
 This follows from Theorem \ref{WBIC_theorem}.  Indeed, by the weak law of large numbers
 \begin{align*}
  \widehat \lambda_{\mr{WBIC},\varphi}(U) &= \frac {n\beta}{T} \sum_{j=1}^T (G(y_j) - G(w^*)) \\
  &\to {n\beta}\bb E_{n,\beta}(G(w)|_U - G(w^*))
 \end{align*}
 where the expectation value is taken with respect to the generalized tempered posterior to $U$: $y_j \sim \Omega_{n,\beta}|_U$.  Theorem \ref{WBIC_theorem}  then tells us that this expected value converges to $\lambda(U)$ in probability as $n \to \infty$, as required.  Here the estimator is defined with the population regret $G$, while Theorem \ref{WBIC_theorem} concerns $G_n$; the difference $n\beta\,\bb E_{n,\beta}\left[(G - G_n)(w) - (G - G_n)(w^*)\right]$ is, by the standard form of Proposition \ref{omega_zero_prop} or \ref{rfv_standard_form} after the substitution $s = n\beta u^{2k}$, of order $\sqrt{\beta}$, and tends to zero since $\beta(n) \to 0$.
\end{proof}

\subsection{The Annealed Posterior}
Let's discuss a variant of the WBIC-based LLC estimator with the same limiting properties as we saw in Proposition \ref{WBIC_estimator_convergence}.  We will replace the tempered posterior in our definition with the \emph{annealed posterior} (see \citealp{SST} for a discussion of this concept from its origin in statistical mechanics).
\begin{definition}
 Fix a constant $n\beta > 0$ and an analytic prior distribution $\varphi$ on $W$.  The \emph{annealed posterior distribution} $\Omega$ is the probability distribution on $W$ defined by
 \begin{align*}
  \Omega^{\mr{ann}}_{n\beta}(U) &= \frac {Z^{\mr{ann}}_{n\beta}(U)}{Z_{n\beta}(W)} \\
  \text{where } Z^{\mr{ann}}_{n\beta}(U) &= \int_U \exp(-n\beta G(w)) \phi(w) \d w
 \end{align*}
 for open sets $U \sub W$.
\end{definition}

\begin{remark}
 We have used the term $n\beta$ for the constant appearing in the definition suggestively and for comparison with the tempered posterior, but we observe that the annealed posterior does not depend on the values of $n,\beta$ separately.
\end{remark}

\begin{prop} \label{annealed_bound_prop}
 Fix a log resolution and a chart $U$ with local coordinate $u$ as in Theorem \ref{general_form_importance_weight_theorem}.  Then
 \[e^{-\frac{\beta \|\xi_n\|^2_\infty}2} Z^{\mr{ann}}_{\frac {3n\beta}2}(\varpi(U)) \le  Z_{n, \beta}(\varpi(U)) \le e^{\frac{\beta \|\xi_n\|^2_\infty}2} Z^{\mr{ann}}_{\frac {n\beta}2}(\varpi(U)).\]
\end{prop}

\begin{proof}
 This follows from Proposition \ref{omega_zero_prop} and the Cauchy--Schwartz inequality: writing $n\beta G_n(\varpi(u)) = n\beta u^{2k} + \beta\sqrt n\, u^k \xi_n(u)$ and using $|\beta \sqrt n\, u^k \xi_n(u)| \le \frac 12 n\beta u^{2k} + \frac 12 \beta\, \xi_n(u)^2$.
\end{proof}

We may therefore define a parallel version of the WBIC LLC estimator using the annealed posterior instead of the usual tempered posterior with good control of its behaviour in the large $n\beta$ limit.

\begin{definition}
Fix an analytic prior distribution $\varphi$ on $W$ and let $y_1, \ldots, y_T$ be independent random variables from the restriction of the annealed posterior to $U$: $y_j \sim \Omega^{\mr{ann}}_{n\beta}|_U$.  Define the \emph{annealed WBIC LLC estimator} by
\[\widehat \lambda^{\mr{ann}}_{\mr{WBIC},\varphi}(U) = \frac {n\beta}{T} \sum_{j=1}^T (G(y_j) - G(w^*)).\]
\end{definition}

The following is then a consequence of Proposition \ref{WBIC_estimator_convergence} and the bound in Proposition \ref{annealed_bound_prop}.

\begin{corollary} \label{annealed_WBIC_estimator_convergence}
 The annealed WBIC LLC estimator $\widehat \lambda^{\mr{ann}}_{\mr{WBIC},\varphi}(U)$ converges in probability to $\lambda(U)$:
 \[\lim_{n \to \infty} \lim_{T \to \infty}  \widehat \lambda^{\mr{ann}}_{\mr{WBIC}}(U)\to_P \lambda(U)\]
 for all choices $\varphi$ of positive analytic prior.
\end{corollary}

\subsection{Soft Localization to an Open Set}
For our next simplification we'll deal with the difficulty of directly imposing the restriction of samples to an open neighborhood $U$ of $w^*$.  We'll deal this in practice by judicious choice of the prior $\varphi$ on $W$, following \citet[\S 4.2]{LLC}.  Choose an embedding $\iota \colon W \inj \RR^D$ for some natural number $D$ and define $\varphi_{\sigma^2}$ by restricting a $D$-dimensional Gaussian centered at $\iota(w^*)$ with variance $\sigma^2$ to $W$.

\begin{remark}
 In \citet{LLC} the prior is instead parameterized by a fundamental scale parameter: the reciprocal of the variance $\gamma = \sigma^{-2}$.  We'll avoid this notation to avoid conflict with the discount factor occurring in the reward function of a Markov decision problem.
\end{remark}

\begin{definition} \label{localized_WBIC_estimator_def}
 Fix $n,T \ge 1$, a local minimum $w^*$ of the loss and a variance $\sigma^2 > 0$.  Denote by $\widehat \lambda_{\sigma^2}(w^*)$ the random variable $\widehat \lambda^{\mr{ann}}_{\mr{WBIC},\varphi_{\sigma^2}}(X)$.
\end{definition}

 Since $\varphi_{\sigma^2}$ is still supported everywhere on $W$ Corollary \ref{annealed_WBIC_estimator_convergence} still applies and the large $n\beta$ limiting behaviour of the estimator is independent of $\sigma^2$ in probability.  The rate of convergence, however, depends on the value of $\sigma^2$: for any open neighborhood $U$ of $w^*$, any $n\beta,T$ and any $\eps>0$ we can choose $\sigma^2$ small enough that the samples $y_1, \ldots, y_T$ are all contained in $U$ with probability $>1-\eps$.

\subsection{Estimators from Stochastic Gradient Langevin Dynamics}
We will now discuss the LLC estimators that are practically computable, corresponding to the estimator defined in \citet[\S 4.4]{LLC}.  These estimators are defined using \emph{stochastic gradient Langevin dynamics} \citep[SGLD;][]{WellingTeh} as a substitute for sampling from the generalized posterior associated to the Gaussian prior $\varphi_{\sigma^2}$.

Given an initial point $w^* \in W$ we will define a sequence $w_1, \ldots, w_T$ of points in $W$ as follows.  Fix a step size $\eps>0$.  Then we let $w_1 = w^*$, and for $j > 1$ define
\begin{equation} \label{SGLD_step}
w_j = w_{j-1} + \frac \eps 2 \left( -\frac{n \beta}{m} \sum_{i=1}^m \nabla_w f(w,w_{j-1},x_{j,i})|_{w=w_{j-1}} + \frac 1{\sigma^2}(w^* - w_{j-1}) \right) + \sqrt{\eps}\, \eta_j
\end{equation}
where $m$ is a natural number (the \emph{batch size}), $x_{j,1}, \ldots, x_{j,m}$ are a sequence of independent samples from $q_{w_{j-1}}$, and $\eta_1, \eta_2, \ldots$ are independent standard Gaussian vectors $\eta_j \sim \mc N(0, I)$ in the ambient $\RR^D$.

Let us comment briefly on the motivation behind this idea. \citet{RobertsTweedie} study the \emph{Langevin diffusion} stochastic differential equation associated to a probability distribution $\pi$ on $\RR^D$
\[\d L_t = \frac 12 \nabla \log \pi(L_t) \d t + \d W_t\]
where $W_t$ is a standard Brownian motion on $\RR^D$.  They argue that under a continuity condition for $\pi$ together with a mild growth condition along rays the Langevin diffusion equation admits $\pi$ as its unique invariant distribution (Theorem 2.1 in \emph{loc. cit.}).  Therefore, in our applications, the generalized (annealed) posterior distribution from which we need to be sampled can be obtained by flowing with respect to the Langevin diffusion.

In order to reify this idea the idea is to use a suitable discrete approximation to the continuous Langevin diffusion: this is the role played by Welling and Teh's stochastic gradient Langevin descent.

\begin{remark}
 We mention here some open problems related to the behaviour of SGLD in singular learning theory; we refer the reader to \citet{HitchcockHoogland} for a more detailed discussion.
 \begin{enumerate}
  \item SGLD with constant step size is known, even under strong convexity assumptions on the objective function, to have limiting behaviour that differs from that of Langevin flow \citep[as shown, for instance, in][]{Pitfalls}.  Therefore the SGLD based estimator is likely to be asymptotically biased with bias depending on the step size $\eps$.  Can this bias be quantified?
  \item If one uses decaying step size then there are results implying almost sure convergence to the posterior distribution as the chain length $T \to \infty$ \citep{TTV}.  These results, however, make assumptions that do not typically hold for the neural network models we are most interested in \citep[for a detailed discussion, see][Apx.\ B]{HitchcockHoogland}.  What can one say upon relaxing these assumptions?
 \end{enumerate}
\end{remark}

\begin{remark}
Notice that in our implementation of SGLD we sample trajectories at each step in the chain \emph{on-policy}.  In other words in equation \ref{SGLD_step} we use the derivative $\nabla_w f(w,w_{j-1},x_{j,i})|_{w=w_{j-1}}$ at step $j$ and not $\nabla_w f(w,w^*,x_{j,i})|_{w=w_{j-1}}$.  Recall that these two random variables have the same expected value but generally different higher moments.  Either choice would provide an estimator for the gradient flow term in the Langevin diffusion for the annealed posterior; we choose to use on-policy sampling since it typically has lower variance.
\end{remark}

If we use a subsequence $w_B, \ldots, w_T$ for some initial point $B>0$ (the \emph{burn-in} parameter) as a substitute for sampling from the posterior distribution $\exp(-n\beta G_n(w))\varphi_{\sigma^*}(w) \d w$ on $W$ then we obtain the estimator that we compute in practice by analogy with the one defined in Definition \ref{localized_WBIC_estimator_def}.  For convenience we summarize the hyperparameters that this estimator depends on a table below.

\begin{table}[!h]
\centering
\label{SGLD_hparam_table}
 \begin{tabular}{|c|c|}
  \hline
  Notation & Description \\ \hline
  $w^* \in W$ & Initial point in $W$\\
  $\beta \in \RR_{>0}$ & Inverse temperature\\
  $n \in \bb N$ & Number of training steps\\
  $\sigma^2 \in \RR_{>0}$ & Noise variance \\
  $T \in \bb N$ & SGLD chain length\\
  $\eps \in \RR_{>0}$ & SGLD step size\\
  $m \in \bb N$ & SGLD batch size\\
  $B \in \bb N$ & Number of burn-in steps\\
  \hline
 \end{tabular}
 \caption{A list of the hyperparameters on which the random variable $\widehat \lambda_{\mr{SGLD}}(w^*)$ depends.  Note that the estimator does not depend on $n$ and $\beta$ independently, only on their product $n\beta$.}
\end{table}

\begin{definition}
 The \emph{SGLD based LLC estimator} $\widehat \lambda_{\mr{SGLD}}(w^*)$ is defined by
\[\widehat \lambda_{\mr{SGLD}}(w^*) = \frac {n\beta}{T - B + 1} \sum_{j=B}^T (G_n(w_j) - G_n(w^*)).\]
\end{definition}

\section{Further LLC Estimate Data} \label{LLC_plot_appendix}

In this appendix we present our LLC estimates during training for all 25 independent seeds with $\alpha=0.6, \gamma=0.98$.  We present two ensembles of figures.
\begin{itemize}
    \item \cref{fig:llc_grid_alpha_0.5_l2_norms} shows LLC estimates for 25 independently sampled models.  We also show the behavior of the weight norms through training.  We observe that the weight norms increase monotonically as the model trains and do not exhibit the "staircase" behaviour of the LLC estimates, nor the decline within phases that we typically observe for LLC estimation.
    \item \cref{fig:llc_grid_alpha_0} shows LLC estimates for models where estimation was performed with a \emph{distinct} value of $\alpha$ to that used in training (as described in Remark \ref{off_dist_rmk}), specifically with $\alpha=0$.  In other words during LLC estimation we only consider the behavior of the model in states where the goal is in the top-left corner.  Note that in such states we cannot distinguish between phases 2 and 3 by means of the regret anymore: these phases are both optimal and receive zero regret.
\end{itemize}
The shaded areas indicate phases, whose detection was described in Section \ref{subsec:auto_phase_detect}.  The blue, beige and pink shaded regions indicate phases 1, 2b and 3 respectively.  

We also show in green an additional phase that appears infrequently.  \emph{Phase 2c} exhibits optimal behaviour exactly in the half space where the mouse $x$-position in greater than or equal to the cheese $x$-position (or, alternatively, the identical condition with the $y$-coordinate).  For the purpose of phase detection this is characterized by the region in policy space where $\pi(R) = 0$ in all states, $\pi(D) = 0$ if the mouse is not to the left of the cheese and not above the cheese, and $\pi(U) = 0$ if the mouse is not to the left of the cheese and not below the cheese.

\begin{remark}
Observe that when estimating LLCs off distribution the transition between phase 2 and phase 3 is frequently visible through a spike in the LLC estimate, or a rapid change in the derivative, even though the policy does not distinguish between these two phases.  The LLC estimates within the two phases however exhibit much higher variance both within and between seeds than we saw in the case of on-distribution LLC estimation.  It may be possible to ameliorate this through more judicious choice of hyperparameters but doing so is not within the scope of the present work.
\end{remark}

\begin{remark}[Environment complexity]
By enhancing our toy model slightly one could explore the ability of the LLC estimator to detect out-of-distribution generalization, meaning here performance in classes of states not seen during training.  For a simple example consider the extension of the model presented here to maze environments (as in the original cheese-in-the-corner example presented in \citealp{Misgen}).  One can allow the distribution of initial states used during training and LLC estimation to differ, but now changes in the LLC may detect changes in performance in out-of-distribution states.
\end{remark}

\begin{figure}[p]
    \centering
    \includegraphics[width=0.85\linewidth]{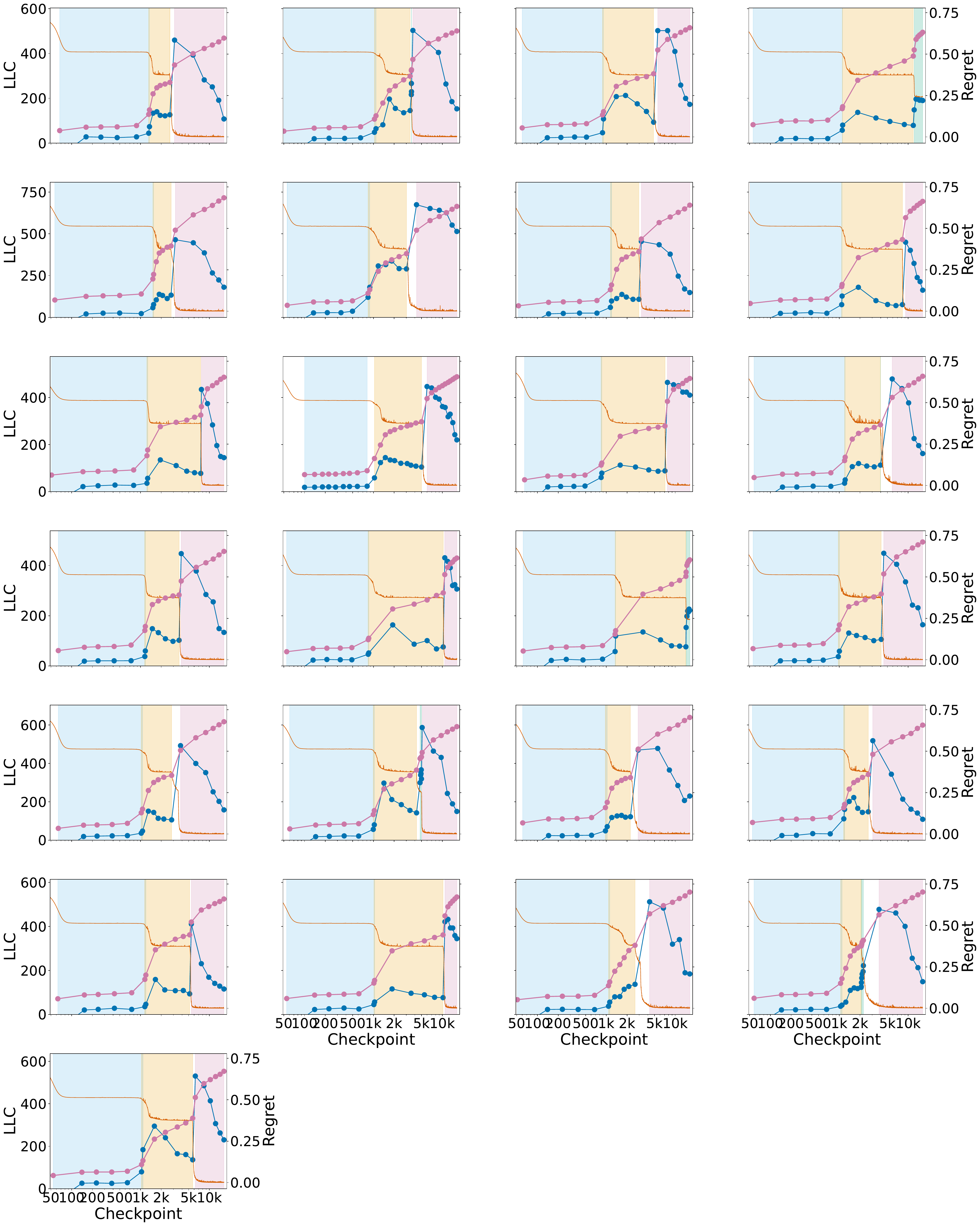}
    \caption{LLC estimates (blue curves) for models with $\gamma=0.98$ and $\alpha=0.6$.These plots also show the weight $L^2$ norms (orange curves), which monotonically increase throughout training, and the regret curves in pink. The weight norms range from 70 to 77.}
    \label{fig:llc_grid_alpha_0.5_l2_norms}
\end{figure}

\begin{figure}[p]
    \centering
    \includegraphics[width=0.85\linewidth]{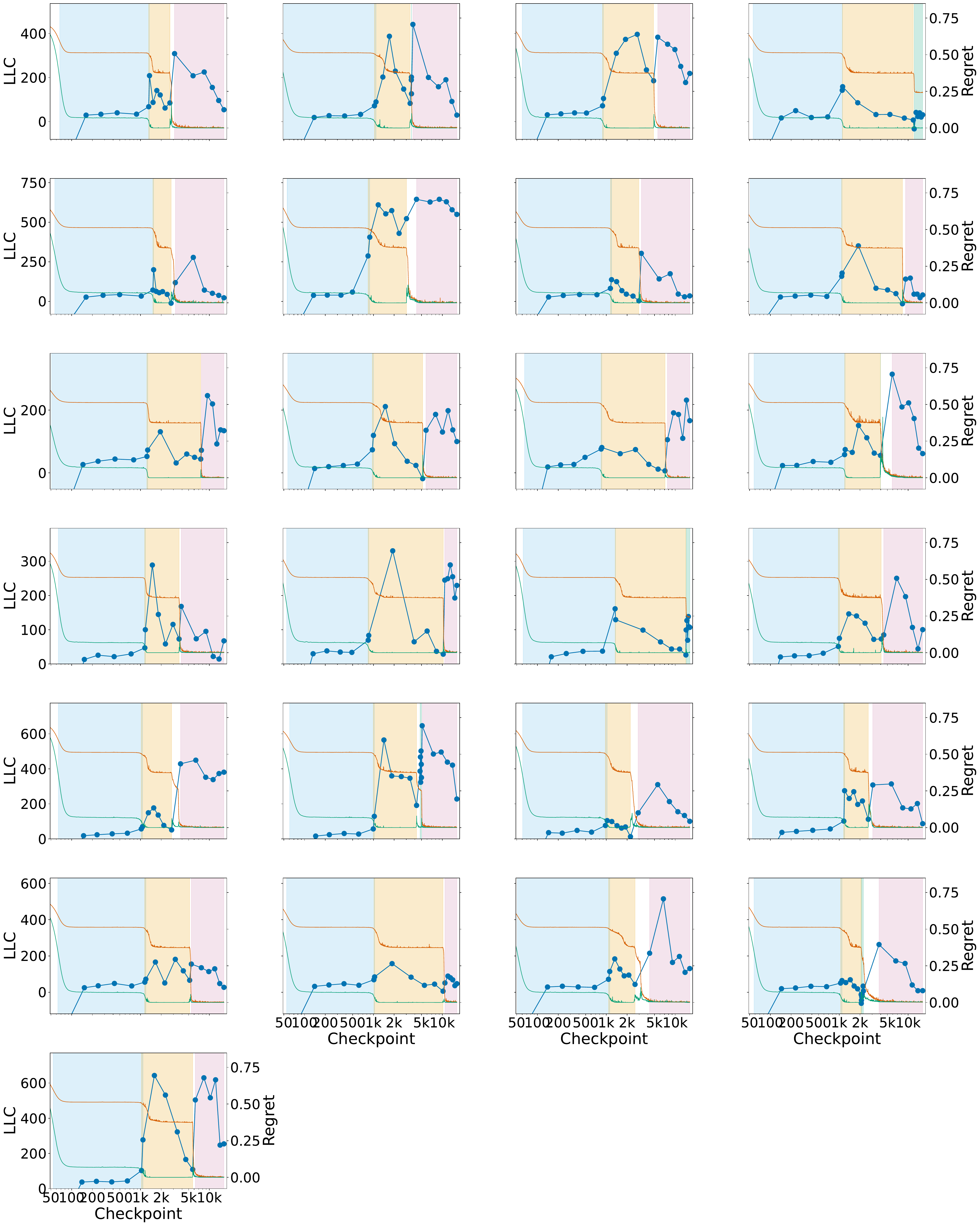}
    \caption{LLC (blue line) estimated with $\gamma=0.98$ and $\alpha=0$ for models trained with values of $\alpha$ and $\gamma$ indicated in the titles of the panels. The runs that share $\alpha$ and $\gamma$ differ only by the random seed with which they were trained. The brown curve shows the on-distribution regret and the green curve shows the regret measured with $\alpha=0$.}
    \label{fig:llc_grid_alpha_0}
\end{figure}

\section{Nonlinear Relationship Between LLC Estimates and Regret} \label{LLC_linearity_appendix}
\begin{figure}[htbp]
    \centering
    \includegraphics[width=0.85\linewidth]{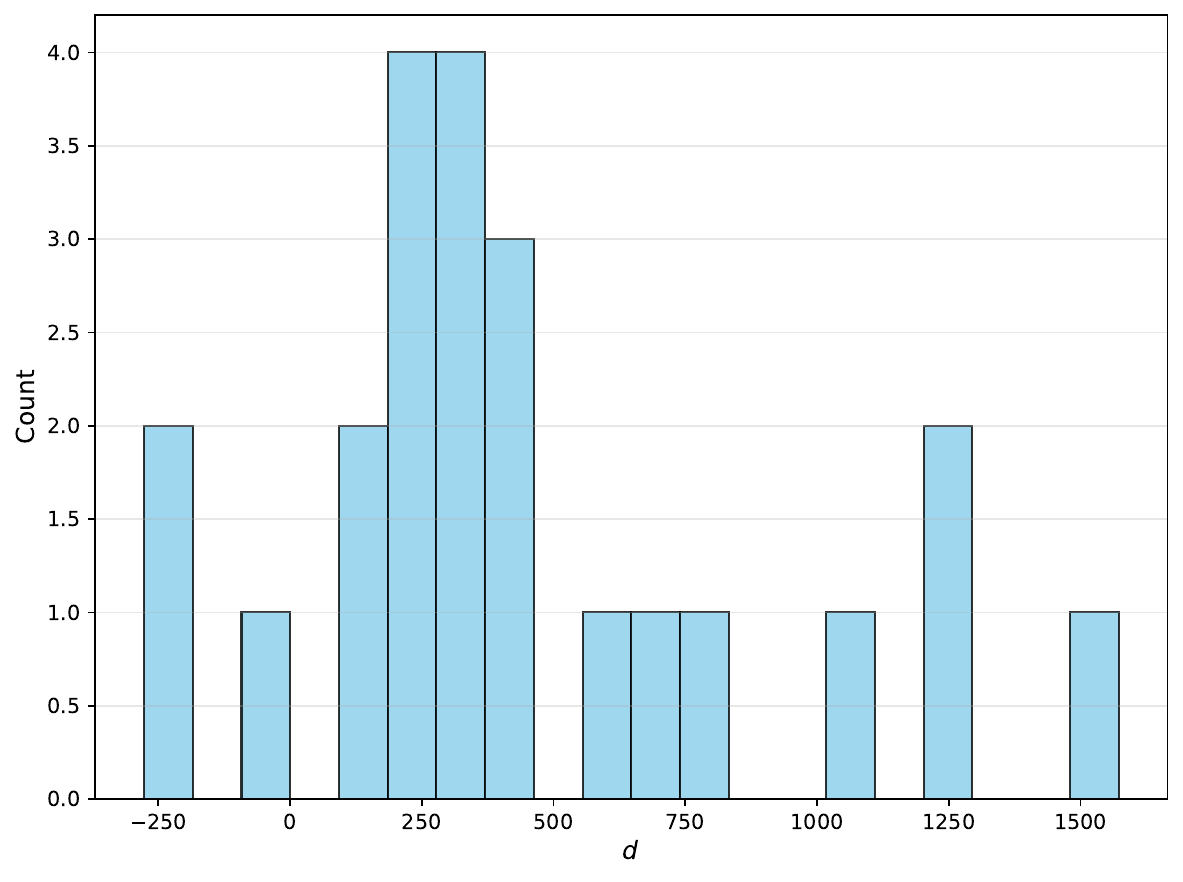}
    \caption{Distribution of the random variable $d$ representing the difference in the rate of change of the LLC estimate with respect to the regret between phases 1 and 2b, and between phases 2b and 3, measured across the 23 iid runs, of 25, for which $\gamma=0.98$ and $\alpha=0.6$ and all three phases are detected.}
    \label{fig:d-distribution}
\end{figure}

Looking at Figure \ref{fig:llc_grid_alpha_0.5_l2_norms} one might be concerned that the LLC is approximately a linear function of the regret.  To see that this is not the case it is sufficient to show that the average (regret, LLC estimate) pairs in the three phases are not collinear (although we do additionally observe non-linear development \emph{within} the phases).  To test this we define the random variable 
\[d = \frac{\hat \lambda_{2b}-\hat \lambda_1}{G_{2b}-G_1} - \frac{\hat \lambda_3-\hat \lambda_{2b}}{G_3-G_{2b}}\]
where $\hat \lambda_i$ and $G_i$ are the average LLC estimates and average regrets among checkpoints in phase $i$ respectively.  In the computation the average return at the checkpoints, the optimal return minus the regret, stands in for $G_i$; this changes the sign of $d$, and \cref{fig:d-distribution} shows $d$ in that convention. Under the hypothesis that $G_i$ and $\hat \lambda_i$ are linearly related the random variable $d$ would have expected value zero for all $\alpha$ and $\gamma$ values.  We may test this using the mean of $d$ across a set of independent samples with fixed $\alpha$ and $\gamma$. We have 23 runs with $\alpha=0.6$ and $\gamma=0.98$ and show the distribution of $d$ for llcs computed on-distribution in Figure \ref{fig:d-distribution}. Since these samples are iid we may perform a one sample t-test and reject the hypothesis that $\bb E(d)=0$ with significance $p=9.9\times 10^{-5}$ ($t = 4.74$ on 22 degrees of freedom).

\section{Sensitivity of Automatic Phase Detection} \label{app:delta_range}

In \Cref{subsec:auto_phase_detect} we described a natural function by which one can automatically detect when the training process has entered a phase by measuring the distance in policy space between the current policy $\pi$, and the closest example policy $\pi'$ contained in some subspace $P$.
If the distance is below some fraction $\delta$ of its maximum value then we say that $\pi$ has approximately entered the phase associated with the subspace of policies $P$.  In other words we detect the phase associated to $P$ when
$$
\min_{\pi' \in P} ||\pi - \pi'|| < \delta d_P
$$
where $d_P = \max_{\pi} \max_{\pi' \in p} ||\pi - \pi'||$ is the maximum distance from the subspace $P$.

\begin{figure}[htbp]
    \centering
    \includesvg[width=0.8\linewidth]{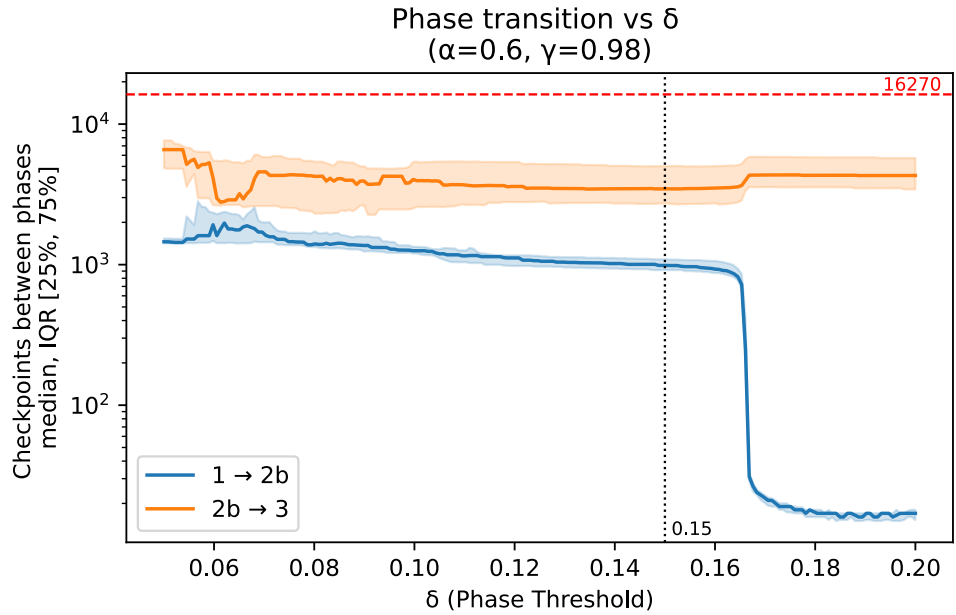}\\[0.4em]
    \caption{A plot showing the median number of checkpoints between Phase 1 and Phase 2b (blue), and the median number of checkpoints
    between Phase 2b to Phase 3 (orange). Models were trained with hyperparameters $\alpha=0.6, \gamma = 0.98$.
    Training runs are only included if the policy was detected to be in all of phase 1, phase 2b and phase 3 at some point during training.}
    \label{fig:delta_range}
\end{figure}

From \Cref{fig:delta_range} we can see that for $0.08 \leq \delta \leq 0.16$ the number of checkpoints between phase 2b and phase 3 is roughly constant, within 25\% of its value at $\delta = 0.15$, while the number between phase 1 and phase 2b grows as $\delta$ decreases, by up to 43\% at $\delta = 0.085$, and stays within 10\% only for $0.12 \leq \delta \leq 0.16$ (recall that spacing between checkpoints is logarithmic).  In particular our LLC estimates within phases 2b and 3 are not significantly sensitive to the value of the threshold $\delta$ within this range.

When $\delta$ exceeds 0.16 we spuriously detect phase 2b early---our threshold has exceeded the distance between the two phases. Since we only include training runs where all three phases are present, for values of $\delta$ sufficiently small the sample size begins to decrease as fewer samples exceed the threshold for detecting the middle phase at some point during training; this is why the plots are not monotonic. Below $\delta=0.05$ no phases are detected in any of the runs, and the average is undefined.

\end{document}